\documentclass[a4paper,12pt]{article}
\usepackage{geometry}
\usepackage[font=small,format=plain,labelfont=bf,up,textfont=it,up]{caption}
\usepackage[utf8]{inputenc}
\usepackage{times}
\usepackage{subfigure} 
\usepackage{graphicx}
\usepackage{microtype}
\usepackage{listings}
\usepackage{amsmath}
\usepackage{amsfonts}
\usepackage{amsthm}
\usepackage{amssymb}

\usepackage{natbib}

\usepackage[page]{appendix}

\usepackage[english]{babel}
\usepackage{fancyhdr}

\newtheorem{theorem}{Theorem}
\newtheorem{lemma}{Lemma}
\newtheorem{corollary}{Corollary}
\theoremstyle{definition}
\newtheorem{definition}{Definition}
\newtheorem{proposition}{Proposition}

\newcommand{\rr}{\mathbf{r}}

\newcommand{\D}{\Delta}

\newcommand{\V}[1]{\text{Var($#1$)}}

\newcommand{\cond}{\, | \,}

\newcommand{\mechIm}{\mathcal I} 
\newcommand{\measSet}{I} 
\newcommand{\dataSpace}{\mathcal X} 
\newcommand{\data}{\mathbf x} 
\newcommand{\dataNeigh}{\mathbf x'} 
\newcommand{\dataSize}{N} 
\newcommand{\tempDataSize}{N_0} 
\newcommand{\batch}{S} 
\newcommand{\batchNeigh}{S'} 
\newcommand{\batchSize}{b} 
\newcommand{\batchCommon}{S\setminus x_N} 

\newcommand{\appropto}{\mathrel{\vcenter{
  \offinterlineskip\halign{\hfil$##$\cr
    \propto\cr\noalign{\kern2pt}\sim\cr\noalign{\kern-2pt}}}}}

\usepackage{mathtools}
\usepackage{authblk}

\usepackage{xcolor}
\usepackage{hyperref}
\hypersetup{%
  linkbordercolor=white,
  citebordercolor=white
}

\title{Differentially Private Markov Chain Monte Carlo}
\date{}

\author{\textbf{Mikko A. Heikkil\"a}$^{*1}$, \textbf{Joonas J\"alk\"o}$^{*2}$, \textbf{Onur Dikmen}$^{3}$, and \textbf{Antti Honkela}$^{1,4,5}$\\
  {$^*$ Equal contribution} \\
  $^1$ Helsinki Institute for Information Technology HIIT,\\
    Department of Mathematics and Statistics, University of Helsinki, Finland \\
  $^2$ Helsinki Institute for Information Technology HIIT,\\
    Department of Computer Science, Aalto University, Finland \\
  $^3$ School of Information Technology, University of Halmstad, Sweden\\
  $^4$ Helsinki Institute for Information Technology HIIT,\\
    Department of Computer Science, University of Helsinki \\
  $^5$ Department of Public Health, University of Helsinki
  }

\begin{document}

\maketitle

\begin{abstract}

    Recent developments in differentially private (DP) machine
  learning and DP Bayesian learning have enabled learning under
  strong privacy guarantees for the training data subjects.
  In this paper, we further extend the applicability of DP Bayesian
  learning by presenting the first general DP Markov chain Monte Carlo (MCMC) 
  algorithm whose privacy-guarantees are not subject to unrealistic assumptions on
  Markov chain convergence and that is applicable to posterior inference in arbitrary models.  Our
  algorithm is based on a decomposition of the Barker acceptance test
  that allows evaluating the Rényi DP privacy cost of the
  accept-reject choice.  We further show how to improve the DP
  guarantee through data subsampling and approximate acceptance tests.

\end{abstract}

\section{Introduction}

Differential privacy (DP) \citep{dwork_et_al_2006, dwork_roth_2014}
and its generalisations to concentrated DP \citep{Dwork2016,Bun2016}
and Rényi DP \citep{Mironov_2017} have recently emerged as the
dominant framework for privacy-preserving machine learning. There are
DP versions of many popular machine learning algorithms, including highly popular
and effective DP stochastic gradient descent (SGD) \citep{Song2013}
for optimisation-based learning.

There has also been a fair amount of work in DP Bayesian machine
learning, with the proposed approaches falling to three main
categories: i) DP perturbation of sufficient statistics for inference in
exponential family models
\citep[e.g.][]{Zhang2015,Foulds2016,Park_2016,Bernstein2018}, ii) gradient
perturbation similar to DP SGD for stochastic gradient Markov chain
Monte Carlo (MCMC) and variational inference
\citep[e.g.][]{Wang2015ICML,Jalko_2016,Li2017}, and iii) DP guarantees for sampling
from the exact posterior typically realised using MCMC
\citep[e.g.][]{Dimitrakakis2014,Zhang2015,Geumlek_2017}.

None of these provide fully general solutions: i) sufficient
statistic perturbation methods are limited to a restricted set of
models, ii) stochastic gradient methods lack theoretical convergence
guarantees and are limited to models with continuous variables, iii) 
posterior sampling methods are applicable to general models, 
but the privacy is conditional 
on exact sampling from the posterior, which is usually impossible to
verify in practice.

In this paper, we present a new generic DP-MCMC method with strict,
non-asymptotic privacy guarantees that hold independently of the
chain's convergence. Our method is based on a recent Barker acceptance
test formulation \citep{Seita_2017}.

\subsection{Our contribution}

We present the first general-purpose DP MCMC
method with a DP guarantee under mild assumptions on the target
distribution.  We mitigate the privacy loss induced by the basic method 
through a subsampling-based approximation. 
We also improve on the existing method of \cite{Seita_2017} for subsampled MCMC, 
resulting in a significantly more accurate method for correcting 
the subsampling induced noise distribution.


\section{Background}
\label{sec:background}


\subsection{Differential privacy}

\begin{definition}[Differential privacy]
	\label{dp_definition}
	A randomized algorithm $\mathcal{M} : \dataSpace^\dataSize \rightarrow \mechIm$ satisfies 
	$(\epsilon, \delta)$ differential privacy, if for all adjacent datasets $\data, \dataNeigh \in \dataSpace^\dataSize$ and
	for all measurable $\measSet \subset \mechIm$ it holds that
	\begin{align}
		\Pr(\mathcal{M}(\data) \in \measSet) \leq e^\epsilon\Pr(\mathcal{M}(\dataNeigh) \in \measSet) + \delta.
	\end{align}
\end{definition}
Adjacency here means that $|\data|=|\dataNeigh|,$ and $\data$ differs from $\dataNeigh$ by a single element, e.g. 
by a single row corresponding to one individual's data in a data matrix.

Recently \citet{Mironov_2017} proposed a Rényi divergence \citep{Renyi_1961} based relaxation for differential
privacy called \textit{Rényi differential privacy} (RDP).

\begin{definition}[Rényi divergence]
	\label{renyi_divergence}
	Rényi divergence between two distributions $P$ and $Q$ defined over $\mechIm$ is defined as
	\begin{align}
		D_\alpha(P \, || \, Q) = 
		\frac{1}{\alpha-1}\log \mathbb{E}_{P} \left[ \left(\frac{p(X)}{q(X)}\right)^{\alpha-1} \right].
	\end{align}
\end{definition}

\begin{definition}[Rényi differential privacy]
	\label{rdp_definition}
	A randomized algorithm $\mathcal{M} : \dataSpace^\dataSize \rightarrow \mechIm$ is
	$(\alpha, \epsilon)$-RDP, if for all adjacent datasets $\data, \dataNeigh$
	it holds that
	\begin{align}
		D_\alpha(\mathcal{M}(\data) \, || \, \mathcal{M}(\dataNeigh)) \leq \epsilon \stackrel{\Delta}{=} \epsilon(\alpha).
	\end{align}
\end{definition}

Like DP, RDP has many useful properties such as invariance to post-processing. 
The main advantage of RDP compared 
to DP is the theory providing tight bounds for doing adaptive compositions, 
i.e., for combining the privacy losses from several possibly adaptive mechanisms accessing the same data, 
and subsampling \citep{Wang_2018}. RDP guarantees can always be converted to $(\epsilon,\delta)$-DP guarantees. These 
existing results are presented in detail in the Supplement.


\subsection{Subsampled MCMC using Barker acceptance}
\label{sec:non_dp_mcmc_theory}

The fundamental idea in standard MCMC methods \citep{Brooks_2011} is that a distribution $\pi(\theta)$ that can only be evaluated
up to a normalising constant,
is approximated by samples $\theta_1, \dots, \theta_t$ drawn from a suitable Markov chain. 
Denoting the current parameter values by $\theta$, the next value is generated using a proposal $\theta'$ drawn from a 
proposal distribution $q(\theta' |\theta)$. An acceptance test is used to 
determine if the chain should move to the proposed value or stay at the current one. 

Denoting the acceptance probability by $\alpha(\theta',\theta)$, a test that satisfies detailed balance \linebreak
$\pi(\theta) q(\theta' |\theta) \alpha(\theta',\theta)~=~\pi(\theta') q(\theta |\theta') \alpha(\theta,\theta') $ 
together with ergodicity of the chain are 
sufficient conditions to guarantee asymptotic convergence to the correct invariant 
distribution $\pi(\theta)$. 
In Bayesian inference, we are typically interested in sampling from the posterior distribution, i.e., 
$\pi(\theta) \propto p(\data|\theta) p (\theta).$ 
However, it is computationally infeasible to use 
e.g. the standard Metropolis-Hastings (M-H) test \citep{Metropolis_1953, Hastings_1970} 
with large datasets, since each iteration would require evaluating $p( \data | \theta)$.

To solve this problem, \citet{Seita_2017} formulate an approximate test that only uses a 
fraction of the data at each iteration. In the rest of this Section we briefly rephrase their 
arguments most relevant for our approach without too much details. A more in-depth treatment 
is then presented in deriving DP MCMC in Section \ref{sec:DP_MCMC}.

We start by assuming the data are exchangeable, so $p(\data | \theta ) = \prod_{x_i \in \data} p( x_i | \theta )$. 
Let 
\begin{equation}
	\label{eq:full_data_Delta}
	\Delta(\theta', \theta) = \sum_{x_i \in \data} \log \frac{ p(x_i | \theta')  }{ p(x_i | \theta) } + 
	\log \frac{ p(\theta') q(\theta | \theta') }{ p(\theta) q(\theta' | \theta) },
\end{equation}
where we suppress the parameters for brevity in the following, and let $V_{log} \sim \text{Logistic}(0,1)$. 
Instead of using the standard M-H acceptance probability $\min\{ \exp(\Delta), 1 \}$, 
\cite{Seita_2017} use a form of Barker acceptance test \citep{Barker_1965} to show that 
testing if 
\begin{equation}
	\label{eq:full_data_log_test}
	\Delta + V_{log} > 0
\end{equation}
also satisfies detailed balance. 
To ease the computational burden, we now want to use only a random subset $\batch
\subset \data$ of size $\batchSize$ instead of full data of size $\dataSize$ to evaluate acceptance. Let
\begin{equation}
	\label{eq:batch_data_Delta}
	\Delta^* (\theta', \theta) = \frac{\dataSize}{\batchSize} \sum_{x_i \in \batch} \log \frac{  p(x_i | \theta') }{ p(x_i | \theta) } + 
	\log \frac{ p(\theta') q(\theta | \theta') }{ p (\theta) q(\theta' | \theta) }.
\end{equation}

Omitting the parameters again, $\Delta^*$ is now an unbiased estimator for $\Delta,$ and assuming $x_i$ are iid 
samples from the data distribution, 
$\Delta^*$ has approximately normal distribution by the Central Limit Theorem (CLT).

In order to have a test that approximates the exact full data test \eqref{eq:full_data_log_test}, 
we decompose the logistic noise as $V_{log} \simeq V_{norm} + V_{cor},$ where $V_{norm}$ has a normal 
distribution and $V_{cor}$ is a suitable correction. Relying on the CLT and on this decomposition we write 
$\Delta^* + V_{cor} \simeq \Delta + V_{norm} + V_{cor} \simeq  \Delta + V_{log},$ so given the correction we can 
approximate the full data exact test using a minibatch.


\subsection{Tempering}
\label{sec:temperature_background}

When the sample size $\dataSize$ is very large, one general problem in Bayesian inference is that the posterior 
includes more and more details. This often leads to models that are much harder to interpret 
while only marginally more accurate than simpler models (see e.g. \citealt{Miller_2018}). 
One way of addressing this issue is to scale the log-likelihood ratios in 
\eqref{eq:full_data_Delta} and \eqref{eq:batch_data_Delta}, so instead of 
$ \log p(x_i | \theta) $ we would have $\tau \log p(x_i | \theta) $ with some $\tau.$ 
The effect of scaling with $0 <\tau < 1$ is then to spread the posterior mass more evenly. 
We will refer to this scaling as tempering.

As an interesting theoretical justification for tempering, \citet{Miller_2018} show a relation between tempered 
likelihoods and modelling error. 
The main idea is to take the error between the theoretical pure data and the actual observable data into 
account in the modelling. Denote the observed data with lowercase and errorless random variables with uppercase letters, 
and let $R \sim \text{Exp}(\beta).$ Then using empirical KL divergence as our modelling 
error estimator $d_{\dataSize},$ instead of the standard posterior we are looking for the posterior conditional on the observed 
data being close to the pure data, i.e., we want $p(\theta | d_{\dataSize} (x_{1:\dataSize},X_{1:\dataSize}) < R),$ which is 
called coarsened posterior or \textit{c-posterior}.

\citet{Miller_2018} show that with these assumptions 
\begin{equation}
p(\theta | d_{\dataSize}(x_{1:\dataSize}, X_{1:\dataSize}) < R) \appropto p(\mathbf{x | \theta})^{\xi_N}p(\theta), 
\end{equation}
where $\appropto$ means approximately proportional to, and $\xi_{\dataSize}= 1/(1+\dataSize/\beta)$, 
i.e., a posterior with tempered likelihoods can be interpreted  
as an approximate c-posterior.


\section{Privacy-preserving MCMC}
\label{sec:DP_MCMC}

Our aim is to sample from the posterior distribution of the model parameters while 
ensuring differential privacy. We start in Section \ref{sec:full_data_dp_mcmc} by formulating DP MCMC 
based on the exact full data Barker acceptance presented in Section \ref{sec:non_dp_mcmc_theory}. 
To improve on this basic algorithm, we then introduce subsampling in Section \ref{sec:batch_dp_mcmc}. 
The resulting DP subsampled MCMC algorithm has significantly better privacy 
guarantees as well as computational requirements than the full data version.


\subsection{DP MCMC}
\label{sec:full_data_dp_mcmc}

To achieve privacy-preserving MCMC, we repurpose the decomposition idea mentioned in 
Section \ref{sec:non_dp_mcmc_theory} with subsampling, i.e., 
we decompose $V_{log}$ in the exact test \eqref{eq:full_data_log_test} 
into normal and correction variables. Noting that $V_{log}$ has variance $\pi^2/3,$ 
fix $0 < C < \pi^2/3$ a constant and write 
\begin{equation}
	\label{eq:V_cor_distribution_full_data} 
	V_{log} \simeq \mathcal N (0,C) + V_{cor}^{(C)},
\end{equation}
where $V_{cor}^{(C)}$ is the correction with variance $\pi^2/3-C$. Now testing if 
\begin{equation}
\label{eq:full_data_appr_test}
\mathcal N (\Delta, C) + V_{cor}^{(C)} > 0
\end{equation}
is approximately equivalent to \eqref{eq:full_data_log_test}.

Since \eqref{eq:V_cor_distribution_full_data} holds exactly for no known distribution
$V_{cor}^{(C)}$ with an analytical expression, \citet{Seita_2017} construct an approximation by
discretising the convolution implicit in \eqref{eq:V_cor_distribution_full_data}, and turning
the problem into a ridge regression problem which can be solved easily. 
To achieve better privacy we want to work with relatively large values of $C$ for
which the ridge regression based solution does not give a good approximation. Instead, we propose to
use a Gaussian mixture model approximation, which gives good empirical performance for larger $C$ as
well. The details of the approximation with related discussion can be found in the Supplement.

In practice, if $V_{cor}^{(C)}$ is an approximation, the stationary distribution
of the chain might not be the exact posterior. However, when the approximation
\eqref{eq:V_cor_distribution_full_data} is good, the accept-reject decisions are rarely affected
and we can expect to stay close to the true posterior. Clearly, in the limit of decreasing $C$ 
we recover the exact test \eqref{eq:full_data_log_test}. We return to this topic in Section \ref{sec:batch_dp_mcmc}.

Considering privacy, on each MCMC iteration we access the data only through the log-likelihood ratio $\Delta$ 
in the test \eqref{eq:full_data_appr_test}. 
To achieve RDP, we therefore need a bound for the Rényi divergence between two Gaussians 
$\mathcal N_{\data} = \mathcal {N} (\Delta_{\data},C)$ and $\mathcal N_{\dataNeigh} = \mathcal N (\Delta_{\dataNeigh},C)$ corresponding to neighbouring datasets $\data, \dataNeigh$. 
The following Lemma states the Rényi divergence between two Gaussians: 
\begin{lemma}
	\label{gaussian_rd}
	Rényi divergence between two normals $\mathcal{N}_1$ and $\mathcal{N}_2$ with parameters
	$\mu_1, \sigma_1$ and $\mu_2, \sigma_2$ respectively is
	\begin{equation}
		D_{\alpha}(\mathcal{N}_1 \, || \, \mathcal{N}_2)=	
		\ln \frac{\sigma_2}{\sigma_1} +
		\frac{1}{2(\alpha-1)}\ln\frac{\sigma_2^2}{\sigma_\alpha^2}+ 
		\frac{\alpha}{2}\frac{(\mu_1-\mu_2)^2}{\sigma_\alpha^2},
	\end{equation}
	where $\sigma_\alpha^2 = \alpha\sigma_2^2+(1-\alpha)\sigma_1^2$.
\end{lemma}
\begin{proof}
	See \citep{gil2013renyi} Table 2.
\end{proof}

\begin{theorem}
	\label{full_data_theorem}	
	Assume either 
	\begin{align}
		\label{smoothness_assumption}
		|\log p(x_i \cond \theta')-\log p(x_i \cond \theta)| \leq B
	\end{align}
	or 
	\begin{align}
		\label{smoothness_assumption2}
		|\log p(x_i \cond \theta)-\log p(x_j \cond \theta)| \leq B,
	\end{align}
	for all $x_i, x_j$ and for all $\theta, \theta'$.
	Releasing a result of the accept/reject
	decision from the test \eqref{eq:full_data_appr_test} is 
	$(\alpha, \epsilon)$-RDP with $\label{eq:full_data_rdp}\epsilon = 2\alpha B^2/C$.
\end{theorem}
\begin{proof}
	Follows from Lemma \ref{gaussian_rd}. See Supplement for further details.
\end{proof}
Using the composition property of RDP (see Supplement), it is straightforward to 
get the following Corollary for the whole chain:
\begin{corollary}
	\label{prop:full_chain_rdp}
	Releasing an MCMC chain of\/ $T$ iterations, where at each iteration the 
	accept-reject decision is done using the test \eqref{eq:full_data_appr_test}, 
	satisfies $(\alpha, \epsilon')$-RDP with $\epsilon' = T2\alpha B^2/C$.
\end{corollary}

We can satisfy the condition \eqref{smoothness_assumption} with 
sufficiently smooth likelihoods and a proposal distribution with a bounded domain:
\begin{lemma}
	Assuming the model log-likelihoods are $L$-Lipschitz over $\theta$ and the diameter of the proposal distribution domain
	is bounded by $d_{\theta}$, LHS of \eqref{smoothness_assumption} 
	is bounded by $Ld_\theta$. 
\end{lemma}
\begin{proof}
	\begin{equation}
		\left|\log p(x_i \cond \theta)-\log p(x_i \cond \theta')\right| 
		\leq L |\theta-\theta'| \leq Ld_\theta .
	\end{equation}
\end{proof}
Clearly, when $Ld_\theta \leq B$ we satisfy the condition in Equation \eqref{smoothness_assumption}.

For some models, using a proposal distribution with a bounded domain could affect the ergodicity of the chain. 
Considering models that are not Lipschitz or using an unbounded proposal distribution, 
we can also satisfy the boundedness condition \eqref{smoothness_assumption} by clipping the log-likelihood 
ratios to a suitable interval.


\subsection{DP subsampled MCMC}
\label{sec:batch_dp_mcmc}

In Section \ref{sec:full_data_dp_mcmc} we showed that we can release samples from the MCMC algorithm
under privacy guarantees. However, as already discussed, evaluating the log-likelihood ratios might require too much 
computation with large datasets. Using 
the full dataset in the DP MCMC setting might also be infeasible for privacy reasons: the noise variance
$C$ in Theorem \ref{full_data_theorem} is upper-bounded by the variance of the logistic random
variable, and thus working under a strict privacy budget we might be able to run the chain for only
a few iterations before $\epsilon'$ in Corollary \ref{prop:full_chain_rdp} exceeds our budget. Using
only a subsample $\batch$ of the data at each MCMC iteration allows us to reduce not only the computational
cost but also the privacy cost through privacy amplification \citep{Wang_2018}.

As stated in Section \ref{sec:non_dp_mcmc_theory}, for the subsampled variant according to the CLT we have 
\begin{equation}
	\label{eq:Delta_star}
	\Delta^* = \Delta + \tilde V_{norm},
\end{equation}
where $\tilde V_{norm}$ is approximately normal with some variance $\sigma_{\Delta^*}^2$. 
Assuming 
\begin{equation}
	\label{eq:batch_variance}
	\sigma_{\Delta^*}^2 < C < \pi^2/3
\end{equation}
for some constant $C$, we now reformulate the decomposition \eqref{eq:V_cor_distribution_full_data} as
\begin{equation}
	\label{eq:V_cor_distribution}
	V_{log} \simeq \underset{ \sim \mathcal N (0,C) }{ \underbrace{V_{norm} + V_{nc}} } + V_{cor}^{(C)},  
\end{equation}
where $V_{norm} \sim \mathcal{N}(0, \sigma_{\Delta^*}^2)$ and $V_{nc} \sim \mathcal{N}(0, C-\sigma_{\Delta^*}^2)$.
We can now write 
\begin{equation}
\label{eq:approximations}
\Delta^* + V_{nc} + V_{cor}^{(C)} \simeq \Delta + V_{norm} + V_{nc} + V_{cor}^{(C)}  \simeq  \Delta + V_{log},
\end{equation}
where the first approximation is justified by the CLT, and the second by the decomposition \eqref{eq:V_cor_distribution}. 
Therefore, testing if 
\begin{equation}
\label{eq:batch_data_log_test}
\mathcal N (\Delta^*, C - \sigma_{\Delta^*}^2 ) + V_{cor}^{(C)} > 0
\end{equation}
approximates the exact full data test \eqref{eq:full_data_log_test}.

As in Section \ref{sec:full_data_dp_mcmc}, the approximations used for arriving at the test \eqref{eq:batch_data_log_test} imply  
that the stationary distribution of the chain need not be the exact posterior. However, we can expect to stay  
close to the true posterior when the approximations are good, since the result only changes if the binary accept-reject decision is 
affected. This is exemplified by the testing in Section \ref{sec:experiments} (see also \citealt{Seita_2017}). 
The quality of the first approximation in \eqref{eq:approximations} depends on the batch size $\batchSize$, 
which should not be too small. As for the second error source, as already noted in 
Section \ref{sec:full_data_dp_mcmc} we markedly improve on this with the GMM 
based approximation, and the resulting error is typically very small (see Supplement). 
In some cases there are known theoretical upper bounds for the total error w.r.t. the true posterior. 
These bounds are of limited practical value since they rely on assumptions that can be hard to meet in general,  
and we therefore defer them to the Supplement.

For privacy, similarly as in Section \ref{sec:full_data_dp_mcmc}, in
\eqref{eq:batch_data_log_test} we need to access the data only for calculating $\Delta^* +
V_{nc}$. Thus, it suffices to privately release a sample from $\mathcal{N}_{\batch} =
\mathcal{N}(\Delta_{\data}^*, C-s_{\Delta_{\data}^*}^2)$, where $s_{\Delta_{\data}^*}^2$ denotes the sample variance 
when sampling from dataset $\data,$ i.e., we need to bound the Rényi divergence between $\mathcal{N}_{\batch}$ 
and $\mathcal{N}_{\batchNeigh}.$ 
We use noise variance $C=2$ in the following analysis.

Next, we will state our main theorem giving an explicit bound that can be used for 
calculating the privacy loss for a single MCMC iteration with subsampling:
\begin{theorem}
	\label{rdp_bound}
	Assuming 
	\begin{align}
		\left|\log p(x_i|\theta')- \log p(x_i|\theta)\right| &\leq \frac{\sqrt{\batchSize}}{\dataSize} \label{1st_assumption}, \\ 
		\alpha &< \frac{\batchSize}{5} \label{2nd_assumption}, 
	\end{align}
	where $\batchSize$ is the size of the minibatch $\batch$ and $\dataSize$ is the dataset size, 
	releasing a sample from $\mathcal N_\batch$ satisfies $(\alpha, \epsilon)$-RDP with
	\begin{align}
		\epsilon = 
		\frac{5}{2\batchSize} + \frac{1}{2(\alpha-1)}  \ln \frac{2\batchSize}{\batchSize - 5 \alpha}
		+ \frac{2\alpha}{\batchSize-5\alpha}.
	\end{align}
\end{theorem}
\begin{proof}
The idea of the proof is straightforward: we need to find an upper bound for each of the terms 
in Lemma \ref{gaussian_rd}, which can be done using standard techniques. Note that for $C=2$ \eqref{1st_assumption} implies that 
the variance assumption \eqref{eq:batch_variance} holds. See Supplement for the full derivation.
\end{proof}

Using the composition \citep{Mironov_2017} and subsampling amplification \citep{Wang_2018} properties of Rényi DP 
(see Supplement), we immediately get the following:
\begin{corollary}
	\label{cor:subsampled_dp_budget}
	Releasing a chain of $T$ subsampled MCMC iterations with sampling ratio $q$, each satisfying 
	$(\alpha, \epsilon(\alpha))$-RDP with $\epsilon(\alpha)$ from Theorem \ref{rdp_bound}, is 
	$(\alpha, T\epsilon')$-RDP with 
	\begin{equation}
		\epsilon' = \frac{1}{\alpha-1} \log \Big ( 
		1+q^2 \binom{\alpha}{2} \min \{ 4(e^{\epsilon(2)}-1), 2e^{\epsilon(2)} \} 
			 + 2\sum_{j=3}^{\alpha} q^j \binom{\alpha}{j} e^{(j-1)\epsilon(j)} \Big ).
	\end{equation}
\end{corollary}
Figures~\ref{fig:eps_vs_q}\ and \ref{fig:eps_vs_T} illustrate how changing the parameters $q$ and $T$ in 
Corollary \ref{cor:subsampled_dp_budget} will affect the privacy budget of DP MCMC.

\begin{figure*}[h!bt]
\subfigure[Privacy vs. subsampling ratio]
{\label{fig:eps_vs_q}
 \includegraphics[width=0.33\textwidth]{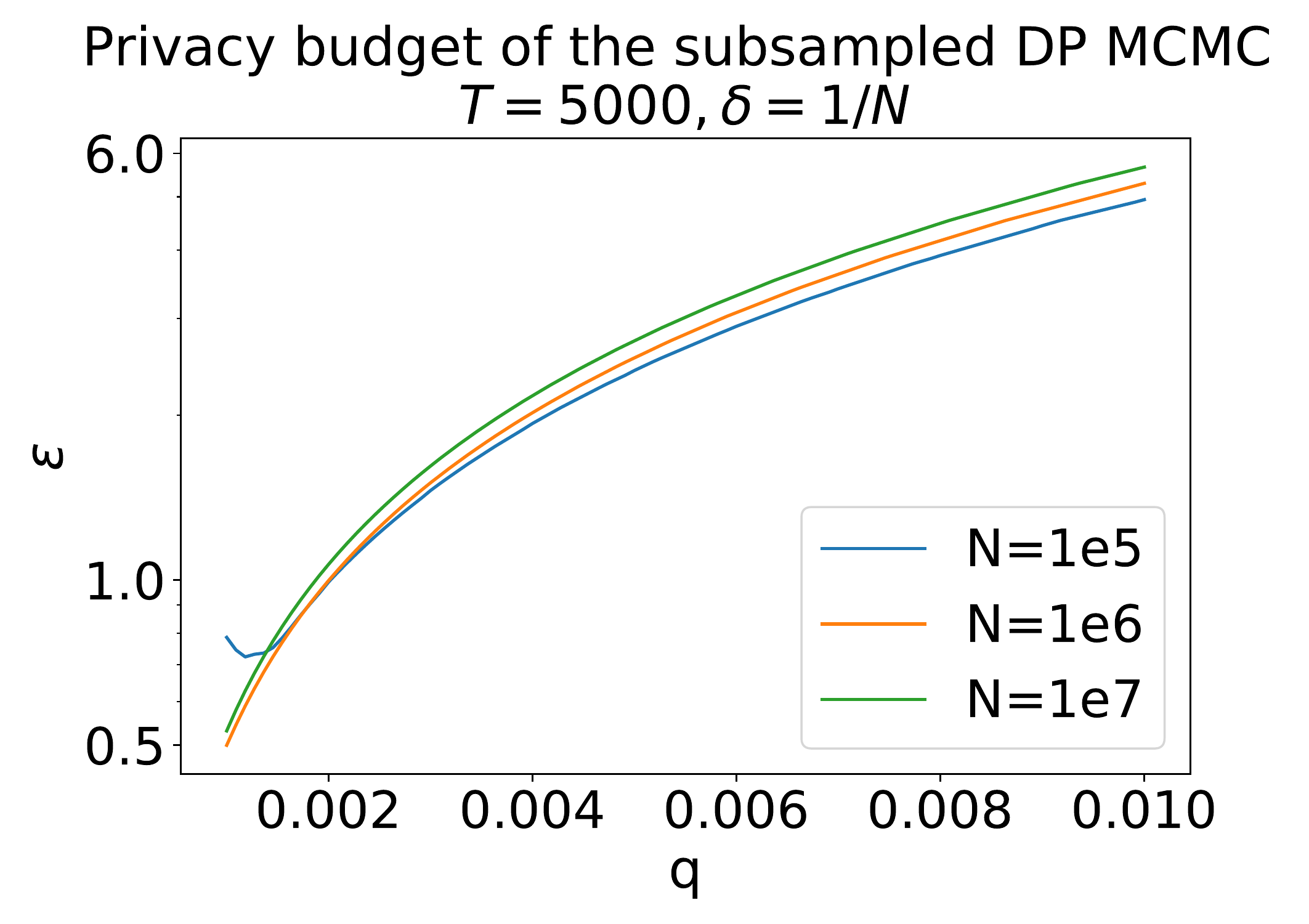}
 }%
 \subfigure[Privacy vs. iterations]
{\label{fig:eps_vs_T}
 \includegraphics[width=0.33\textwidth]{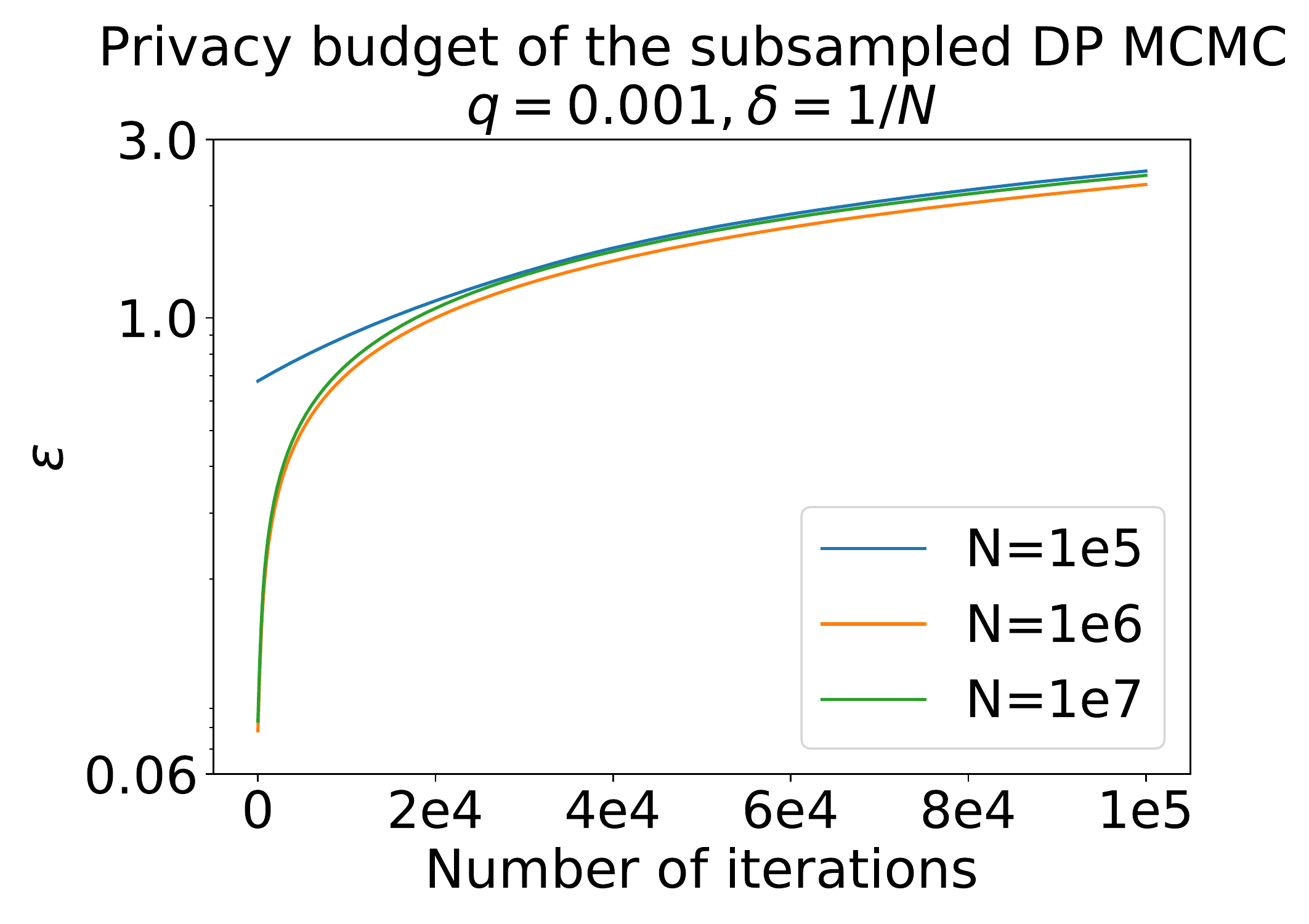}
 }%
\subfigure[Clipping vs. proposal variance]
{\label{fig:clip_vs_var}
 \includegraphics[width=.33\textwidth]{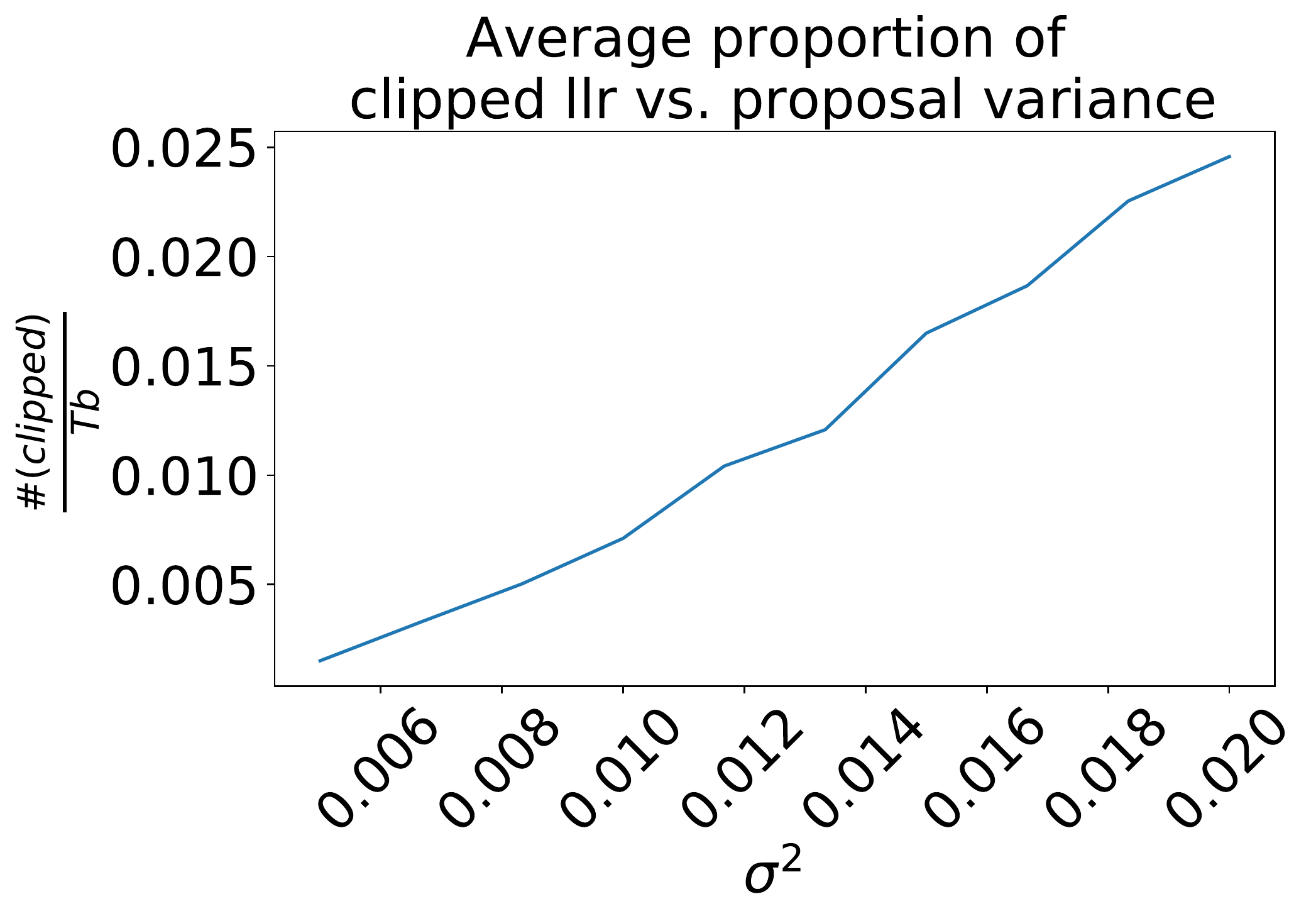}%
 }%
\caption{Parameter effects. Calculating total privacy budget from Corollary 
		\ref{cor:subsampled_dp_budget} for different dataset sizes: in Figure~\ref{fig:eps_vs_q} as a function of subsampling
		ratio, and in Figure~\ref{fig:eps_vs_T} as a function of number of iterations. Figure~\ref{fig:clip_vs_var} shows the 
		proportion of clipped log-likelihood ratios as a function of proposal variance for 
  		the GMM example detailed in Section \ref{sec:experiments}.\label{fig:param_effects}}
\end{figure*}

Similarly as in the full data case in Section \ref{sec:full_data_dp_mcmc}, we can satisfy the condition \eqref{1st_assumption} 
with sufficiently smooth likelihoods or by clipping. Figure \ref{fig:clip_vs_var} shows how frequently we 
need to clip the log-likelihood ratios to maintain the bound in \eqref{1st_assumption} as a function of proposal 
variance using a Gaussian mixture model problem defined in Section \ref{sec:experiments}. Using smaller proposal variance
will result in smaller changes in the log-likelihoods between the previous and the proposed parameter values, 
which entails fewer clipped values.

However, the bound in \eqref{1st_assumption} gets tighter with increasing $\dataSize$. To counterbalance
this, either the proposals need to be closer to the current value (assuming suitably smooth
log-likelihood), resulting in a slower mixing chain, or $\batchSize$ needs to increase, affecting privacy
amplification. For very large $\dataSize$ we would therefore like to temper the log-likelihood ratios in a
way that we could use sufficiently small batches to benefit from privacy amplification, while
still preserving sufficient amount of information from the likelihoods and reasonable mixing
properties. Using the c-posterior discussed in Section \ref{sec:temperature_background} with
parameter $\beta$ s.t. $\tempDataSize = \dataSize \beta/(\beta+\dataSize)$, instead of condition \eqref{1st_assumption} we then
require
\begin{align}
	\label{eq:tempered_fraction}
	\left|\log p(x_i|\theta')-\log p(x_i|\theta)\right| &\leq \frac{\sqrt{\batchSize}}{\tempDataSize},
\end{align}
which does not depend on $\dataSize.$


\section{Experiments}
\label{sec:experiments}

In order to demonstrate our proposed method in practice, we use a simple 2-dimensional Gaussian mixture model\footnote{The code for running all the experiments will be made freely available.}, 
that has been used by \citet{Welling_2011} and \citet{Seita_2017} in the non-private setting:
\begin{align}
\theta_j & \sim \mathcal N (0, \sigma_j^2,), \quad j= 1,2 \\
x_i & \sim 0.5 \cdot \mathcal N (\theta_1, \sigma_x^2) + 0.5 \cdot \mathcal N (\theta_1+\theta_2, \sigma_x^2),
\end{align}
where $\sigma_1^2 = 10, \sigma_2^2 = 1, \sigma_x^2=2 .$ For the observed data, we use fixed parameter values $\theta = (0, 1). $ 
Following \citet{Seita_2017}, we generate $10^6$ samples from the model to use as training data. We use $\batchSize=1000$ for the 
minibatches, and adjust the temperature of the chain s.t. $\tempDataSize=100$ in \eqref{eq:tempered_fraction}. This corresponds to the temperature 
used by \citet{Seita_2017} in their non-private test.

If we have absolutely no idea of a good initial range for the parameter values, especially in higher 
dimensions the chain might waste the privacy budget in moving towards areas with higher posterior probability. 
In such cases we might want to initialise the chain in at least somewhat reasonable location, 
which will cost additional privacy. 
To simulate this effect, we use the differentially private variational inference (DPVI) introduced 
by \citet{Jalko_2016} with a small privacy budget $(0.21, 10^{-6})$ to find a rough estimate 
for the initial location.

As shown in Figure \ref{fig:MoG_figure}, the samples from the tempered chain with DP are nearly 
indistinguishable from the samples drawn from the non-private tempered chain. 
Figure \ref{fig:accuracy_figure} illustrates how the accuracy is affected by privacy. 
Posterior means and variances are computed from the first $t$ iterations of the private 
chain alongside the privacy cost $\epsilon,$ which increases with $t.$ The baseline is given 
by a non-private chain after 5000 iterations. The plots 
show the mean and the standard error of the mean over 20 runs.

\begin{figure*}[htb]
\vskip -0.15in
\subfigure[Samples without privacy]
{\label{fig:MoG_figure_nonprivate}
 \includegraphics[width=0.4\textwidth]{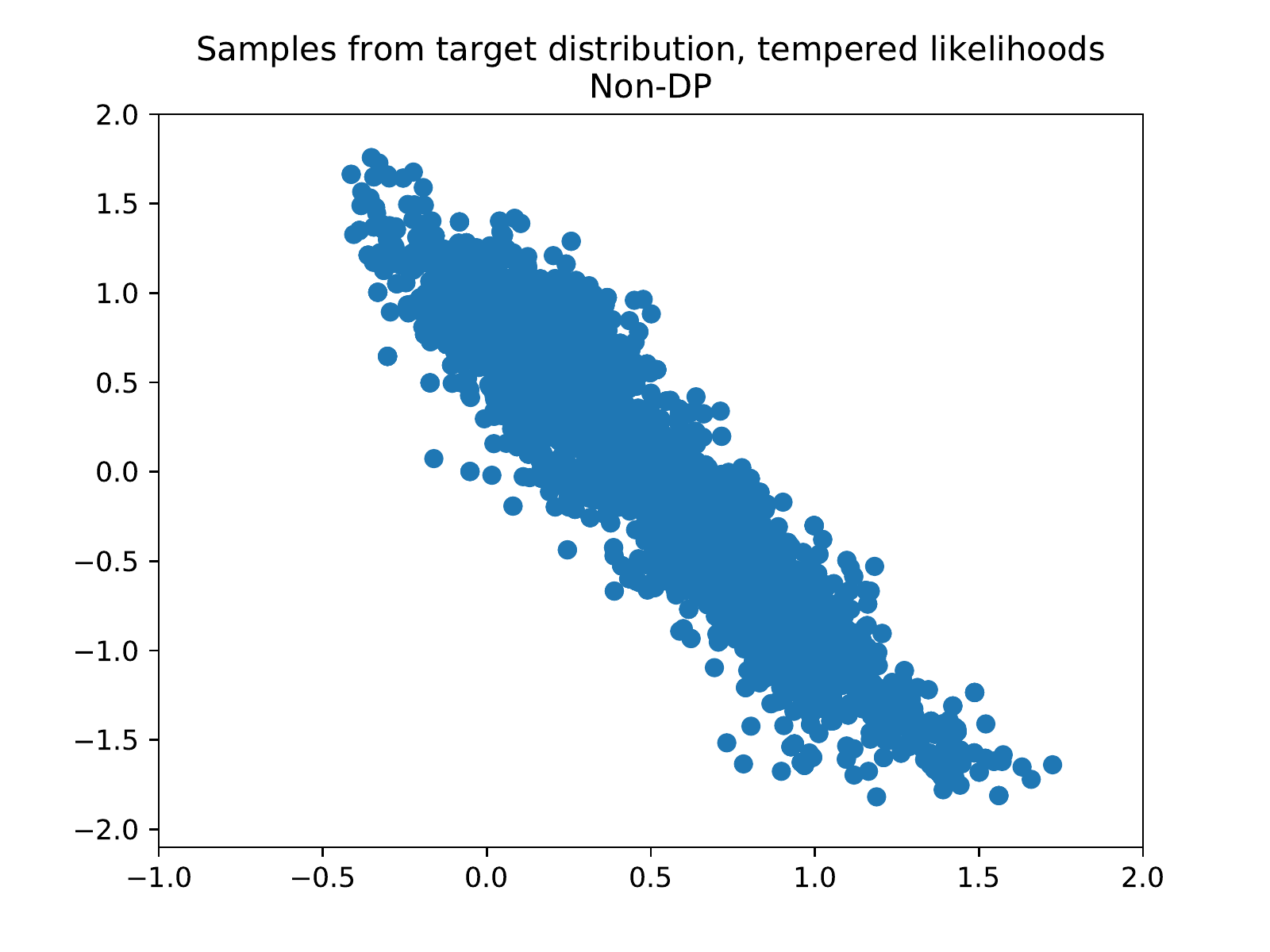}
 }
 \hfill
 \subfigure[Samples with DP]
{\label{fig:MoG_figure_private}
 \includegraphics[width=0.4\textwidth]{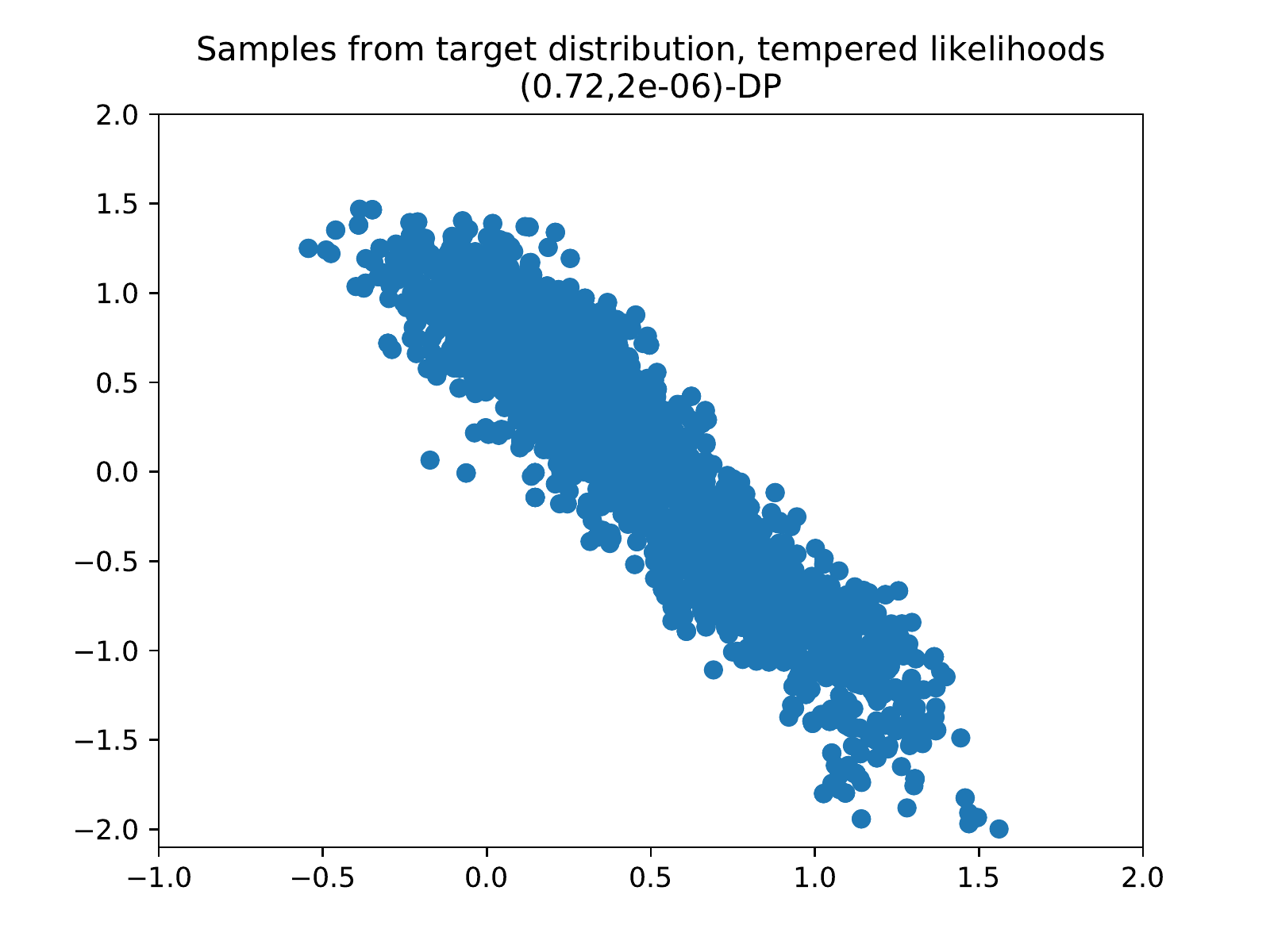} 
 }
  \caption{ Results for the GMM  experiment with tempered likelihoods: \ref{fig:MoG_figure_nonprivate} shows 5000 samples from the chain without privacy 
  and \ref{fig:MoG_figure_private} with privacy. The results with strict privacy are very close to the non-private results.
   }
  \label{fig:MoG_figure}
  \vskip -0.25in
\end{figure*}

\begin{figure*}[h!bt]
\subfigure[Posterior mean accuracy]
{\label{fig:mean_acc}
 \includegraphics[width=0.4\textwidth]{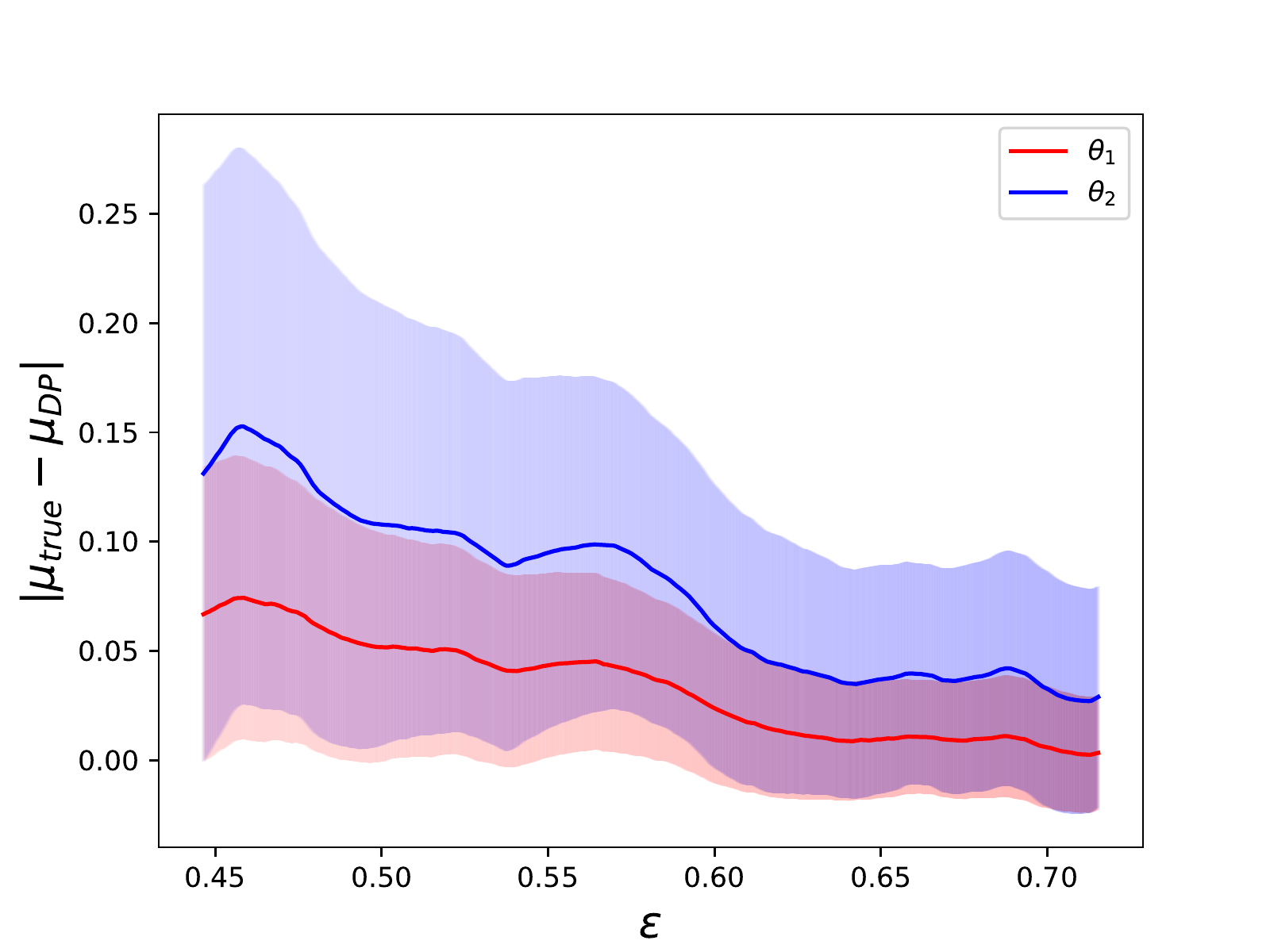}
 }
 \hfill
 \subfigure[Posterior variance accuracy]
{\label{fig:var_acc}
 \includegraphics[width=0.4\textwidth]{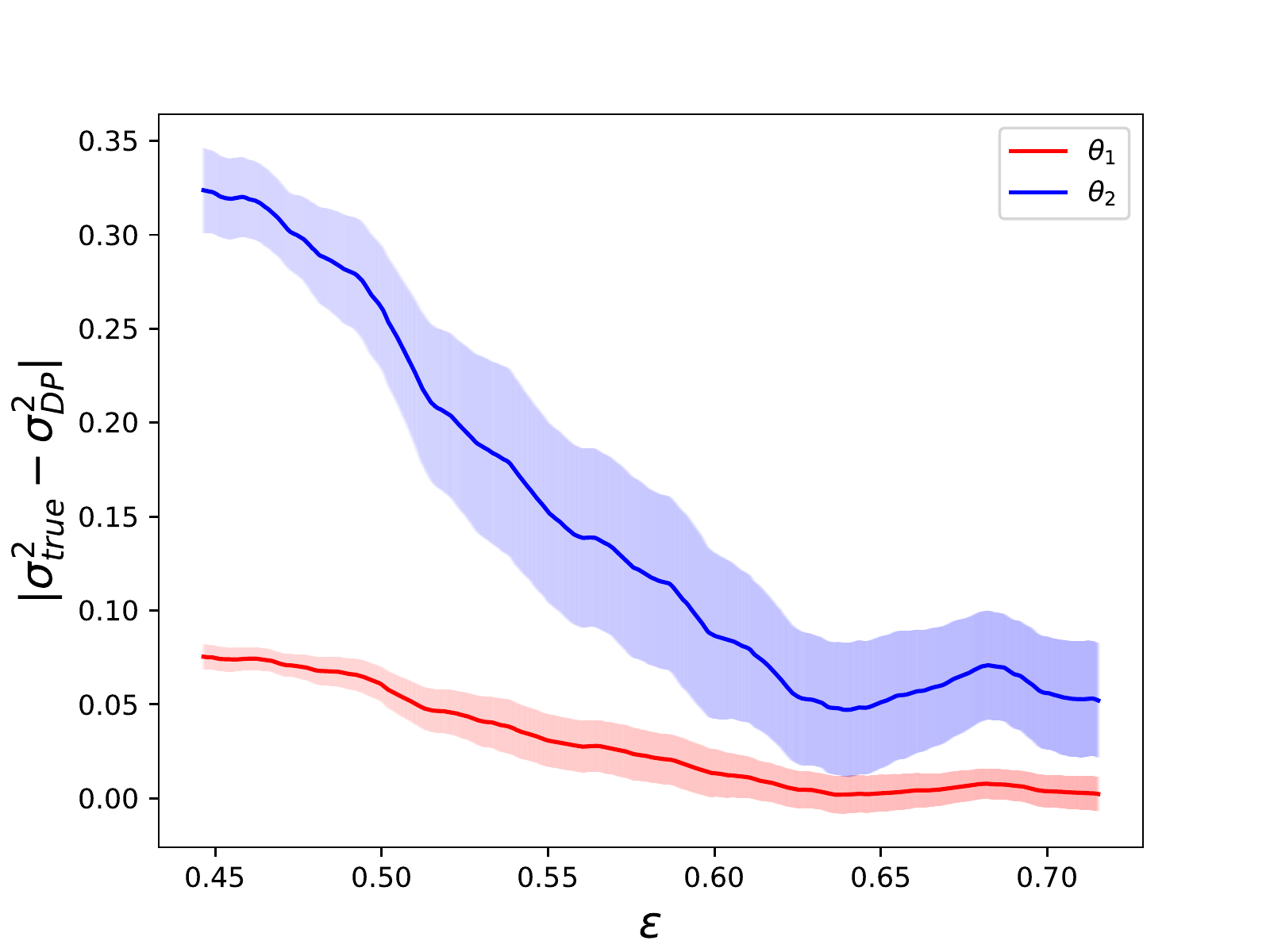}
 }
	\caption{Intermediate private posterior statistics starting from iteration $T = 1000$ compared against the 
	baseline given by a non-private chain after $5000$ iterations. Solid lines show the mean error between 20 runs of the algorithm
	with errorbars illustrating the standard error of the mean between the runs.}
\label{fig:accuracy_figure}
\vskip -0.25in 
\end{figure*}


\section{Related work}

Bayesian posterior sampling under DP has been studied using several
different approaches. \citet{Dimitrakakis2014} note that drawing a
single sample from the posterior distribution of a model where the
log-likelihood is Lipschitz or bounded yields a DP guarantee. The
bound on $\epsilon$ can be strengthened by tempering the posterior by
raising the likelihood to a power $\tau \in (0, 1)$ to obtain the
tempered posterior
\begin{equation}
  \label{eq:tempered_posterior}
  \pi_{\tau}(\theta) \propto p(\theta) p(\data \mid \theta)^{\tau}.
\end{equation}
The same principle is discussed and extended by \citet{Wang2015ICML},
\citet{Zhang2015} and \citet{Dimitrakakis2017} in the classical DP
setting and by \citet{Geumlek_2017} in the RDP setting.
\citet{Wang2015ICML} dub this the ``one posterior sample'' (OPS)
mechanism. The main limitation of all these methods is that the
privacy guarantee is conditional on sampling from the exact posterior,
which is in most realistic cases impossible to verify.

The other most widely used approach for DP Bayesian inference is
perturbation of sufficient statistics of an exponential family model
using the Laplace mechanism. This straightforward application of the
Laplace mechanism was mentioned at least by \citet{Dwork2009} and has
been widely applied since by several authors
\citep[e.g.][]{Zhang2015,Foulds2016,Park_2016,Honkela_2018,Bernstein2018}.
In particular, \citet{Foulds2016} show that the sufficient statistics
perturbation is more efficient than OPS for models where both are
applicable. Furthermore, these methods can provide an unconditional
privacy guarantee. Many of the early methods ignore the Laplace noise
injected for DP in the inference, leading to potentially biased
inference results. This weakness is addressed by
\citet{Bernstein2018}, who include the uncertainty arising from the
injected noise in the modelling, which improves especially the accuracy of
posterior variances for models where this can be done.

MCMC methods that use gradient information such as Hamiltonian Monte
Carlo (HMC) and various stochastic gradient MCMC methods have become
popular recently. DP variants of these were first proposed by
\citet{Wang2015ICML} and later refined by \citet{Li2017} to make use
of the moments accountant \citep{Abadi2016}. The form of the privacy
guarantee for these methods is similar to that of our method: there is
an unconditional guarantee for models with a differentiable Lipschitz
log-likelihood that weakens as more iterations are taken. Because of
the use of the gradients, these methods are limited to differentiable
models and cannot be applied to e.g.\ models with discrete variables.

Before \citet{Seita_2017}, the problem of MCMC without using the full data 
has been considered by many authors 
(see \citealt{Bardenet_2017} for a recent literature survey). 
The methods most closely related to ours are the ones by \citet{Korattikara_2014} and \citet{Bardenet_2014}. 
From our perspective, the main problem with these approaches is the adaptive batch size: 
the algorithms may regularly need to use all observations on a single iteration \citep{Seita_2017}, 
which clashes with privacy amplification. 
\citet{Bardenet_2017} have more recently proposed an improved version 
of their previous technique alleviating the problem, but the batch sizes can 
still be large for privacy amplification.


\section{Discussion}

While gradient-based samplers such as HMC are clearly dominant in the
non-DP case, it is unclear how useful they will be under DP.
Straightforward stochastic gradient methods such as stochastic
gradient Langevin dynamics (SGLD) can be fast in initial convergence
to a high posterior density region, but it is not clear if they can
explore that region more efficiently. HMC does have a clear advantage
at exploration, but \citet{Betancourt2015} clearly demonstrates that
HMC is very sensitive to having very accurate gradients and therefore
a naive DP HMC is unlikely to perform well. We believe that using a
gradient-based method such as DP variational inference
\citep{Jalko_2016} as an initialisation for the proposed method can
yield overall a very efficient sampler that can take advantage of the
gradients in the initial convergence and of MCMC in obtaining accurate
posterior variances. Further work in benchmarking different approaches
over a number of models is needed, but it is beyond the scope of this
work.

The proposed method allows for structurally new kind of assumptions to
guarantee privacy through forcing bounds on the proposal instead of or
in addition to the likelihood. This opens the door for a lot of
optimisation in the design of the proposal. It is not obvious how the
proposal should be selected in order to maximise the amount of useful
information obtained about the posterior under the given privacy
budget, when one has to balance between sampler acceptance rate and
autocorrelation as well as privacy. We leave this interesting question
for future work.


\subsection*{Acknowledgements}

The authors would like to thank Daniel Seita for sharing the original code for their paper. 

This work has been supported by the Academy of Finland
[Finnish Center for Artificial Intelligence FCAI and
grants 294238, 303815, 313124].

\bibliographystyle{abbrvnat}
\bibliography{exactdp}

\clearpage



\appendix

\section*{Supplement}

\section{Useful differential privacy results}

\begin{proposition}
	\label{composition}
	A composition of two RDP algorithms $\mathcal{M}_1$, $\mathcal{M}_2$ with RDP guarantees
	$(\alpha, \epsilon_1)$ and $(\alpha, \epsilon_2)$, is $(\alpha, \epsilon_1+\epsilon_2)$-RDP.
\end{proposition}
\begin{proof}
	See {\citet[Proposition 1]{Mironov_2017} }.
\end{proof}
The next result follows immediately from Proposition \ref{composition}.
\begin{corollary}
	\label{cor:T-fold_composition} 
	Releasing a result from a $T$-fold composition of a $(\alpha, \epsilon)$-RDP query is
	$(\alpha, T\epsilon)$-RDP.
\end{corollary}

The following Proposition states the privacy amplification via subsampling result of \citet{Wang_2018}.
\begin{proposition}	
	\label{prop:RDP_subsampling}
	A randomised algorithm $\mathcal{M}$ which accesses the whole dataset $\data$ only through 
	subset $\batch$ of the dataset and satisfies $(\alpha, \epsilon)$-RDP w.r.t. to $\batch$,
	is $(\alpha, \epsilon')$-RDP with
	\begin{align*}
		\epsilon' \leq\, \frac{1}{\alpha-1} & \log  
		\Big ( 1+q^2 \binom{\alpha}{2} 
		 \cdot \min \Big \{ 4(e^{\epsilon(2)}-1), e^{\epsilon(2)}
		 \min \Big \{ 2,(e^{\epsilon(\infty)-1})^2\Big \} \Big \} \\
		& + \sum_{j=3}^\alpha q^j {\alpha \choose j} e^{(j-1)\epsilon(j)}\min\left\{2, 
		(e^{\epsilon(\infty)}-1)^j \right\}
		\Big ),
	\end{align*}
	where $q = |\batch|/|\data|,$ and $\alpha \geq 2$ is an integer, and $\epsilon(\infty) = \lim_{j\rightarrow \infty} \epsilon(j).$
\end{proposition}
\begin{proof}
	See {\citet[Theorem 10]{Wang_2018} }.
\end{proof}

Finally, we can convert RDP privacy guarantees back to $(\epsilon, \delta)$-DP guarantees using the following proposition.
\begin{proposition}
	\label{prop:from_RDP_to_DP}
	An $(\alpha, \epsilon)$-RDP algorithm $\mathcal{M}$ also satisfies $(\epsilon', \delta)$-DP
	for all $0<\delta<1$ with 
	\begin{align}
		\epsilon' = \epsilon + \frac{\log (1/\delta)}{\alpha-1}.
	\end{align}
\end{proposition}
\begin{proof}
	See {\citet[Proposition 3]{Mironov_2017} }.
\end{proof}


\section{Proof of main text's Theorem \ref{full_data_theorem}}

Denote the maximally different adjacent datasets by 
$\mathbf x_1, \mathbf x_2.$ The mechanism releases a sample from 
$\mathcal{N}_1 = \mathcal N (\D_1, C),$ and $\mathcal{N}_2 = \mathcal N (\D_2, C),$ where 
$\Delta_1, \Delta_2$ are calculated with $\mathbf x_1, \mathbf x_2,$ respectively.

We want to show that 
\begin{align}
D_{\alpha} (\mathcal N_1 || \mathcal N_2) &= \log \frac{\sigma_1}{\sigma_2} + 
	\frac{1}{2(\alpha-1)} \log \frac{\sigma_2^2}{\alpha \sigma_2^2 + (1-\alpha)\sigma_1^2 } + 
	\frac{\alpha}{2} \frac{(\mu_1 - \mu_2)^2}{\alpha \sigma_2^2 + (1-\alpha)\sigma_1^2 } \\
	&\leq \frac{2 \alpha B^2}{C}
\end{align}
assuming that either
	\begin{align}
		\label{smoothness_assumption_sup1}
		|\log p(x_i | \theta')-\log p(x_i | \theta)| < B \ \forall x_i, \theta, \theta'
	\end{align}
	or 
	\begin{align}
		\label{smoothness_assumption_sup2}
		|\log p(x_i | \theta)-\log p(x_j | \theta)|<B, \ \forall x_i, x_j, \theta .
	\end{align}

\begin{proof}

W.l.o.g., we can assume that the differing element between $\mathbf x_1$ and $\mathbf x_2$ is the final one, so 
$x_{1,i} = x_{2,i}, i=1,\dots,N-1$.

Since $\sigma_1^2 = \sigma_2^2 = C,$ we immediately have
\begin{align}
D_{\alpha} (\mathcal N_1 || \mathcal N_2) &= \log \frac{\sigma_1}{\sigma_2} + 
	\frac{1}{2(\alpha-1)} \log \frac{\sigma_2^2}{\alpha \sigma_2^2 + (1-\alpha)\sigma_1^2 } + 
	\frac{\alpha}{2} \frac{(\mu_1 - \mu_2)^2}{\alpha \sigma_2^2 + (1-\alpha)\sigma_1^2 } \\
	\label{eq:full_data_equal_var}
	& = \frac{\alpha}{2 C} (\mu_1 - \mu_2)^2 \\
	& = \frac{\alpha}{2 C} [\sum_{i=1}^N \log \frac{p(x_{1,i}|\theta') }{ p(x_{1,i}|\theta) } 
		- \sum_{i=1}^N \log \frac{p(x_{2,i}|\theta') }{ p(x_{2,i}|\theta) } ]^2 \\
	\label{eq:full_data_interlude}
	& = \frac{\alpha}{2 C} \left | \log \frac{p(x_{1,N}|\theta') }{ p(x_{1,N}|\theta) } 
		- \log \frac{p(x_{2,N}|\theta') }{ p(x_{2,N}|\theta) } \right |^2 .
\end{align}

Assuming \eqref{smoothness_assumption_sup1}, and continuing from \eqref{eq:full_data_interlude} 
\begin{align}
& \frac{\alpha}{2 C} \left | \log \frac{p(x_{1,N}|\theta') }{ p(x_{1,N}|\theta) } 
		- \log \frac{p(x_{2,N}|\theta') }{ p(x_{2,N}|\theta) }  \right |^2 \\
&\leq \frac{\alpha}{2 C} \Big ( \left | \log \frac{p(x_{1,N}|\theta') }{ p(x_{1,N}|\theta) } \right | 
		+ \left | \log \frac{p(x_{2,N}|\theta') }{ p(x_{2,N}|\theta) } \right | \Big )^2 \\		
&\leq \frac{\alpha}{2 C} \left | 2 B \right |^2 \\
&\leq \frac{2 \alpha B^2}{C}.
\end{align}

On the other hand, assuming \eqref{smoothness_assumption_sup2}, and again continuing from \eqref{eq:full_data_interlude} gives 
\begin{align}
& \frac{\alpha}{2 C} \left | \log \frac{p(x_{1,N}|\theta') }{ p(x_{1,N}|\theta) } 
		- \log \frac{p(x_{2,N}|\theta') }{ p(x_{2,N}|\theta) } \right |^2 \\
&= \frac{\alpha}{2 C} \left | \log \frac{p(x_{1,N}|\theta') }{ p(x_{2,N}|\theta') } 
		- \log \frac{p(x_{1,N}|\theta) }{ p(x_{2,N}|\theta) } \right |^2 \\
&\leq \frac{\alpha}{2 C} \Big ( \left | \log \frac{p(x_{1,N}|\theta') }{ p(x_{2,N}|\theta') } \right |
		+ \left | \log \frac{p(x_{1,N}|\theta) }{ p(x_{2,N}|\theta) } \right | \Big )^2 \\
&\leq \frac{\alpha}{2 C} \left | 2 B \right |^2 \\
&\leq \frac{2 \alpha B^2}{C},
\end{align}
which is the same bound as before.

\end{proof}


\section{Proof of main text's Theorem \ref{rdp_bound}}

The Barker test amounts to checking the following condition:
\begin{align}
	\Delta^* &+ V_{nc} + V_{cor}^{(2)} > 0, \, \text{where} \\
	\label{eq:Delta_star}
	\Delta^* &= \frac{N}{b}\sum_{i \in \batch} \underbrace{\log\frac{p(x_i | \theta') }{p(x_i | \theta)}}_{r_i} + 
		\log\frac{q(\theta | \theta') p(\theta)}{q(\theta' | \theta) p(\theta')}, \\
	V_{nc} &\sim \mathcal{N}(0, 2-s_{\Delta^*}^2),
\end{align}
$N$ is the full dataset size, $b$ is the batch size, $s_{\Delta^*}^2$ is the sample variance, and summation over 
$\batch$ here means summing over the elements in the batch, indexed by the element number $i$.

In other words, with a slight abuse of notation and writing capital letters for random variables the mechanism releases a sample from 
\begin{align}
\label{eq:normal_eq_middle}
\mathcal{N} ( N \bar{\rr}, 2-\V{ \frac{N}{b} \sum_{i\in \batch} R_i }  ) =& \mathcal{N} ( N \bar{\rr}, 2 - \frac{N^2}{b^2} \sum_{i\in \batch} \V{ R} ) \\
\label{eq:normal_eq}
\approx & \mathcal{N}(N \bar{\mathbf r}, 2-\frac{N^2}{b}\V{\rr}),
\end{align}
where \eqref{eq:normal_eq_middle} holds because $R_i$ are conditionally iid with a common distribution written as $R$, 
and $\V{ \rr }$ means the sample variance estimated from the actual iid sample $r_i, i \in \batch$ we 
have, i.e., a vector of length $b$.

Assume that 
\begin{align}
\label{eq:supp_proof_ass1}
	|r_{i} | &\leq \frac{\sqrt{b}}{N} \stackrel{\Delta}{=} c, \forall \ i \ \text{ and } \\ 
\label{eq:supp_proof_ass2}
	\alpha &< \frac{b}{5}.
\end{align}

We want to show that 
\begin{align}
\label{eq:renyi_divergence}
	D_\alpha(\mathcal{N}_1 \, || \, \mathcal{N}_2) &= 
	\underbrace{\ln \frac{\sigma_2}{\sigma_1}}_{f_1} +
	\underbrace{\frac{1}{2(\alpha-1)}\ln\frac{\sigma_2^2}{\alpha\sigma_2^2+(1-\alpha)\sigma_1^2}}_{f_2}+
	\underbrace{\frac{\alpha}{2}\frac{(\mu_1-\mu_2)^2}{\alpha\sigma_2^2+(1-\alpha)\sigma_1^2}}_{f_3} \\
	& \leq \frac{5}{2b} + \frac{1}{2(\alpha-1)}  \ln \frac{2b}{b- 5 \alpha}
		+ \frac{2\alpha}{b-5\alpha}.
\end{align}

\begin{proof}

As a first step, we have 
\begin{align}
\label{eq:var_bounding1}
 0 < \V \rr =& \mathbb E(\rr^2) - \mathbb E(\rr)^2 \leq \mathbb E(\rr^2) = 1/b \sum_{i \in \batch} r_i^2 \leq \frac{b}{N^2} \\
\label{eq:var_bounding2}
 & \Rightarrow 2 - \frac{N^2}{b} \V \rr  \in [1,2),
\end{align}
where the last inequality in \eqref{eq:var_bounding1} follows from \eqref{eq:supp_proof_ass1}.

Denote the maximally different adjacent datasets as 
$\rr_{1}, \rr_{2}$ that produce draws from $\mathcal{N}_1$ and 
$\mathcal{N}_2$ respectively, parameterised with means and variances as in  
\eqref{eq:normal_eq}. W.l.o.g., we can assume that the differing element is the final one, so 
we have $r_{1,i} = r_{2,i}, i=1,\dots,b-1$. We write $i \in \batchCommon$ to index a summation over the batch 
omitting the differing element.

The proof proceeds by bounding each of the terms $f_1, f_2, f_3$ in \eqref{eq:renyi_divergence}.

To start with, $f_1$ can be bounded as follows:
\begin{align}
	\label{eq:f1_abs_vals}
	f_1 &= \frac{1}{2} \ln\frac{\sigma_2^2}{\sigma_1^2} 
	\leq \frac{1}{2} | \ln\frac{\sigma_2^2}{\sigma_1^2} | 
	\leq \frac{1}{2} | \sigma_2^2-\sigma_1^2 | \\
	& = \frac{1}{2}  | 2 - \frac{N^2}{b} \V{\rr_2}- ( 2- \frac{N^2}{b} \V{\rr_1}) | \\
	&= \frac{N^2}{2 b} | 1/b \sum_{i \in \batch} r_{1,i}^2 - ( \bar \rr_{1} )^2 -  1/b \sum_{i \in \batch} r_{2,i}^2 + ( \bar \rr_{2} )^2 |   \\
	&= \frac{N^2}{2 b} | 1/b ( r_{1,b}^2  -  r_{2,b}^2 ) + ( 1/b \sum_{i \in \batch} r_{2,i} )^2 - ( 1/b \sum_{i \in \batch} r_{1,i} )^2 |   \\
	&= \frac{N^2}{2 b^2} | ( r_{1,b}^2  -  r_{2,b}^2 ) + 1/b ( r_{2,b}^2 - r_{1,b}^2 
		+ 2( \sum_{i \in \batchCommon} r_{2,i} \cdot r_{2,b} - \sum_{i \in \batchCommon} r_{1,i} \cdot r_{1,b}  )    ) |   \\
	&= \frac{N^2}{2 b^2} | \frac{b-1}{b} ( r_{1,b}^2  -  r_{2,b}^2 ) - \frac{2}{b}  
		( \sum_{i \in \batchCommon} r_{i} ) ( r_{1,b} -  r_{2,b}   )  |   \\
	&= \frac{N^2}{2 b^3} | (b-1) ( r_{1,b}^2  -  r_{2,b}^2 ) - 
		2 ( \sum_{i \in \batchCommon} r_{i} ) ( r_{1,b} -  r_{2,b}   )  |   \\
	\label{eq:s2/s1_upper_bound}
	&\leq \frac{N^2}{2 b^3} [ (b-1) ( c^2 ) + 2 (b-1) c ( 2c ) ] \\
	& = \frac{N^2}{2 b^3} (b-1) 5 c^2 \\
	\label{eq:s2/s1_rough_upper_bound}
	& \leq \frac{5}{2b}, 
\end{align}
where the final inequality in \eqref{eq:f1_abs_vals} holds because we have \eqref{eq:var_bounding2}, 
and \eqref{eq:s2/s1_upper_bound} as well as the final bound in \eqref{eq:s2/s1_rough_upper_bound} 
follow from \eqref{eq:supp_proof_ass1}.

For the common denominator term 
$\alpha\sigma_2^2+(1-\alpha)\sigma_1^2 $ in $f_2$ and $f_3$, we can first repeat essentially 
the previous calculation to get
\begin{align}
	\sigma_2^2 - \sigma_1^2 & \geq - | \sigma_2^2-\sigma_1^2 | \\
	& = \cdots \\ 
	&= - \frac{N^2}{b^3} | (b-1) ( r_{1,b}^2  -  r_{2,b}^2 ) - 2 ( \sum_{i \in \batchCommon} r_{i} ) ( r_{1,b} -  r_{2,b}   )  |   \\
	& \geq - \frac{N^2}{b^3} [(b-1) c^2 + 2(b-1)c(2c)]  \\
	& = - \frac{N^2}{b^3} (b-1) 5 c^2 \\
	\label{eq:s2/s1_rough_lower_bound}
	& \geq - \frac{5}{b}.
\end{align}

Combining \eqref{eq:s2/s1_rough_lower_bound} and \eqref{eq:var_bounding2} we get
\begin{align}
\alpha\sigma_2^2+(1-\alpha)\sigma_1^2 &= \sigma_1^2 + \alpha(\sigma_2^2 - \sigma_1^2) \\
\label{eq:denominator_rough_bound}
& \geq 1 - \alpha \frac{5}{b} > 0,
\end{align}
where the final inequality follows from \eqref{eq:supp_proof_ass2}.

For the numerator in $f_3$ we have
\begin{align}
	(\mu_1-\mu_2)^2 &= \left(\frac{N}{b}\sum_{i \in \batch} r_{1,i}- \frac{N}{b}\sum_{i \in \batch} r_{2,i} \right)^2 \\
	&= \left( \frac{N}{b} ( r_{1,b} - r_{2,b} ) \right)^2 \\
	\label{eq:f3_nominator_bound}
	& \leq \left( \frac{2N c}{b} \right)^2 \\
	\label{eq:f3_nominator_rough_bound}
	& \leq \frac{4}{b}.
\end{align}

Finally, using the derived bounds in \eqref{eq:s2/s1_rough_upper_bound}, \eqref{eq:denominator_rough_bound}, 
and \eqref{eq:f3_nominator_rough_bound} with the fact that $\sigma_2^2\leq 2$ from \eqref{eq:var_bounding2}, 
the bound for the R\'enyi divergence \eqref{eq:renyi_divergence} becomes
\begin{align}
	D_\alpha(\mathcal{N}_1 \, || \, \mathcal{N}_2) &\leq \frac{5}{2b} + \frac{1}{2(\alpha-1)} ( \ln 2 - \ln(1 - \frac{5 \alpha}{b}) ) + \frac{\alpha}{2} \frac{4}{b} \frac{1}{1-\frac{5 \alpha}{b}} \\
	\label{eq:full_rough_bound}
	& \leq \frac{5}{2b} + \frac{1}{2(\alpha-1)}  \ln \frac{2b}{b- 5 \alpha}
		+ \frac{2\alpha}{b-5\alpha}.
\end{align}

If we instead use the tempered log-likelihoods with temperature $\tau=\frac{N_0}{N}$, the effect is to replace $r_i$ by $\tau r_i.$ 
The same proof then holds when instead of $N$ we write $N_0.$

\end{proof}


\section{Bounding the approximations errors}

As mentioned in the main text, with finite data and $b<N$ the acceptance test \eqref{eq:batch_data_log_test} in the main text is an approximation. 
For this case, there are some known theoretical bounds for the errors induced. The general idea with the following Theorems is that 
by bounding the errors induced by each approximation step, we can 
find a bound on the error in the stationary distribution of the approximate chain w.r.t. the exact posterior. 
The references in this Section mostly point to the main text. The exceptions are obvious from the context.

First, Theorem \ref{thm:Seita_cor_1} gives an upper bound 
for the error due to $\Delta^*$ having approximately normal instead of exactly normal distribution 
as in \eqref{eq:Delta_star}:
\begin{theorem}
\label{thm:Seita_cor_1}
\begin{equation*}
\sup_y | \mathbb P(\Delta^* < y) - \Phi ( \frac{ y - \Delta}{s_{\Delta^*}} ) | \leq \frac{ 6.4 \mathbb E [|Z|^3] + 2 \mathbb E [| Z |]}{\sqrt b},
\end{equation*}
where $Z = N ( \log \frac{p(X|\theta')}{p(X|\theta)} - \mathbb E[ \log \frac{p(X|\theta')}{p(X|\theta)} ]) .$
\end{theorem}
\begin{proof}
See {\citep[Cor. 1]{Seita_2017} }.
\end{proof}

Next, we have a bound for the error in the test quantity \eqref{eq:batch_data_log_test} 
relative to the exact test \eqref{eq:full_data_log_test} given in Theorem \ref{thm:Seita_cor_2}. The 
original proof {\citep[Cor. 2]{Seita_2017} } assumes that $C=1$ and \eqref{eq:V_cor_distribution} holds exactly. 
We present a slightly modified proof that holds for any $C$ and also accounts for the error due to having only an approximate correction to the 
logistic distribution. We start with a helpful Lemma before the actual modified Theorem.
\begin{lemma}
\label{lemma:Seita_convolution}
Let $P(x)$ and $Q(x)$ be two CDFs satisfying $\sup_x |P(x) - Q(x)| \leq \epsilon$ with $x$ in some real range. Let $R(y)$ be the density of another random variable $Y$. 
Let $P'$ be the convolution $P * R$ and $Q'$ be the convolution $Q * R$. Then $P'(z)$ (resp. $Q'(z)$) is the CDF of sum $Z = X+ Y$ of independent random variables $X$ with CDF $P(x)$ (resp. $Q(x)$) and $Y$ with density $R(y)$. Then $$\sup_x |P'(x) - Q'(x)| \leq \epsilon .$$
\end{lemma}
\begin{proof}
See {\citep[Lemma 4]{Seita_2017} }.
\end{proof}

\begin{theorem}
\label{thm:Seita_cor_2}
If $\sup_y | \mathbb P(\Delta^* < y) - \Phi ( \frac{ y - \Delta}{s_{\Delta^*}} ) | \leq \epsilon_1(\theta', \theta, b) $ 
and $\sup_y | S' ( y ) - S( y ) | \leq \epsilon_2,$ 
then $ \sup_y | \mathbb P(\Delta^* + V_{nc} + V_{cor}^{(C)} < y) - S ( y - \Delta ) | \leq \epsilon_1(\theta', \theta, b) + \epsilon_2,$
where $s_{\Delta^*}$ is the sample standard deviation of $\Delta^*,$ $S'$ is the cdf of the approximate logistic distribution 
produced by $\mathcal N (0, C) + V_{cor}^{(C)},$ and $S$ is the exact logistic function.
\end{theorem}
\begin{proof}
As in the original proof {\citep[Cor. 2]{Seita_2017} } the main idea is to use Lemma \ref{lemma:Seita_convolution} two times. 
First, take $P(y) = \mathbb P ( \Delta^* < y ), Q(y) = \Phi ( \frac{y-\Delta}{s_{\Delta^*}} )$ and convolve 
with $V_{nc}$ which has density $\phi (\frac{x}{\sqrt{C - s_{\Delta^*}^2 })}).$ For the second step, take the results
$P'(y) = \mathbb P ( \Delta^* + V_{nc} < y ), Q'(y) = \Phi ( \frac{y-\Delta}{\sqrt C}) $ and convolve 
with the density of $V_{cor}^{(C)}$ to get $P''(y) = \mathbb P ( \Delta^* + V_{nc} + V_{cor}^{(C)} < y ), Q''(y) = S' ( y-\Delta).$ 
By Lemma \ref{lemma:Seita_convolution}, both convolutions preserve the error bound $\epsilon_1(\theta',\theta,b),$ 
and we therefore have 
\begin{align}
& \sup_y | \mathbb P(\Delta^* + V_{nc} + V_{cor}^{(C)} < y) - S ( y - \Delta ) | \\
 & = \sup_y | \mathbb P(\Delta^* + V_{nc} + V_{cor}^{(C)} < y) - S'(y-\Delta) + S'(y-\Delta) - S ( y - \Delta ) | \\
 \label{eq:triangle_ie}
 & \leq \sup_y | \mathbb P(\Delta^* + V_{nc} + V_{cor}^{(C)} < y) - S'(y-\Delta)|+ \sup_y | S' (y) - S ( y ) | \\
& \leq \epsilon_1(\theta',\theta,b) + \epsilon_2,
\end{align}
where \eqref{eq:triangle_ie} follows from the triangle inequality.
\end{proof}

Finally, a bound on the test error implies a bound 
for the stationary distribution of the Markov chain relative to the true posterior, given in 
Theorem \ref{thm:Korattikara_thm_1}. 
Writing $d_v (P,Q)$ for the total variation distance between distributions $P$ and $Q$, 
$\mathcal T_0$ for the transition kernel of the exact Markov chain, $\mathcal S_0$ 
for the exact posterior, and $\mathcal S_{\epsilon}$ for the stationary distribution of the approximate 
transition kernel where $\epsilon$ is the error in the acceptance test, we have:
\begin{theorem}
\label{thm:Korattikara_thm_1}
If $\mathcal T_0$ satisfies the contraction condition $d_v (P \mathcal T_0, \mathcal S_0 ) < \eta d_v (P, \mathcal S_0 )$ 
for some constant $\eta \in [0,1)$ and all probability distributions $P$, then 
\begin{equation*}
d_v ( \mathcal S_0, \mathcal S_{\epsilon} ) \leq \frac{\epsilon}{1-\eta},
\end{equation*}
where $\epsilon$ is the bound on the error in the acceptance test.
\end{theorem}
\begin{proof}
See {\citep[Theorem 1]{Korattikara_2014} }.
\end{proof}

Generally, especially the contraction condition in Theorem \ref{thm:Korattikara_thm_1} 
can be hard to meet: it can be shown to hold e.g. for some Gibbs samplers (see e.g. \citealt[Theorem 6.1]{Bremaud_1999})
but it is not usually valid for an arbitrary model, and even checking the condition might not be trivial.


\section{Numerical approximation of the correction distribution}

As noted in the main text, we need to find an approximate distribution $V_{cor}^{(C)}$ s.t. 
\begin{equation}
V_{log} \stackrel{d}{=} \mathcal N (0,C) + V_{cor}^{(C)}, 
\end{equation}
where $V_{log}$ has a standard logistic distribution. The approximation method of \citet{Seita_2017} casts 
the problem into a ridge regression problem, which can be solved effectively. However, 
nothing constrains the resulting function from having negative values. In order to use it as an approximate pdf, 
\citet{Seita_2017} set these to zeroes and note that as long as $C$ is small enough, such values are rare and 
hence do not affect the solution much. In practice, their solution seems to work very well with 
small values of $C$, e.g. when $C \leq 1.$

Since we want to use larger $C$ for 
the privacy, we propose to approximate $V_{cor}^{(C)}$ with a Gaussian mixture 
model (GMM). Since the result is always a valid pdf, the problem of negative values does not arise.

To find the correction pdf, denote the density of the GMM approximation with $K$ components by 
$\tilde f_{cor}$, the GMM component parameters by $\pi_k, \mu_k$ and $\sigma_k,$
and the standard normal density by $\phi.$ We have 
\begin{align*}
f_{log}(x) &= (f_{norm} * f_{cor}) (x) \simeq (f_{norm} * \tilde f_{cor})(x) \\
& = \int_{\mathbb R} f_{norm}(x) \tilde f_{cor} (x-t) dt \\
& = \int_{\mathbb R} \phi(\frac{x}{\sqrt{C} }) [\sum_{k=1}^K \pi_k \phi (\frac{x-t-\mu_k}{\sigma_k}) ] dt \\
& = \sum_{k=1}^K \pi_k  \phi (\frac{x-\mu_k}{\sqrt{ C+\sigma_k^2} }) 
= \tilde f_{log}^{(C)}(x; \pi_k, \mu_k, \sigma_k, k=1,\dots, K)
\end{align*}

As the logistic pdf is symmetric around zero, 
we require our GMM approximation to be symmetric as well. We achieve this by
creating a counterpart for each mixture component with an opposite sign mean and identical
variance and weight.
To construct the approximation on some interval $[-a,a] \subset \mathbb R$, we discretise 
the interval into $n$ points, and fit the GMM by minimising the loss function
\begin{equation}
\mathcal L (\pi, \mu, \sigma) = \| f_{log} - \tilde f_{log}^{(C)} \|_2
\end{equation}
calculated over the discretisation. 
Since GMM is a generative model, sampling from the optimised approximation is easy.

Figure \ref{fig:x_cor_approx_errors} shows the approximation error $\max_y | S'(y)-S(y) |,$ 
where $S'$ is the approximate logistic ecdf and $S$ the exact logistic cdf, due to $\tilde V_{cor}$ 
using the ridge regression solution proposed by \citet{Seita_2017} and the GMM. 
The error measure is the same as in Theorem \ref{thm:Seita_cor_2} in the Supplement. 
Empirically, as shown in the Figure, we can have noticeably better approximation especially with larger $C$ values. 

\begin{figure}[htb]
\begin{center}
\centerline{\includegraphics[width=.5\columnwidth]{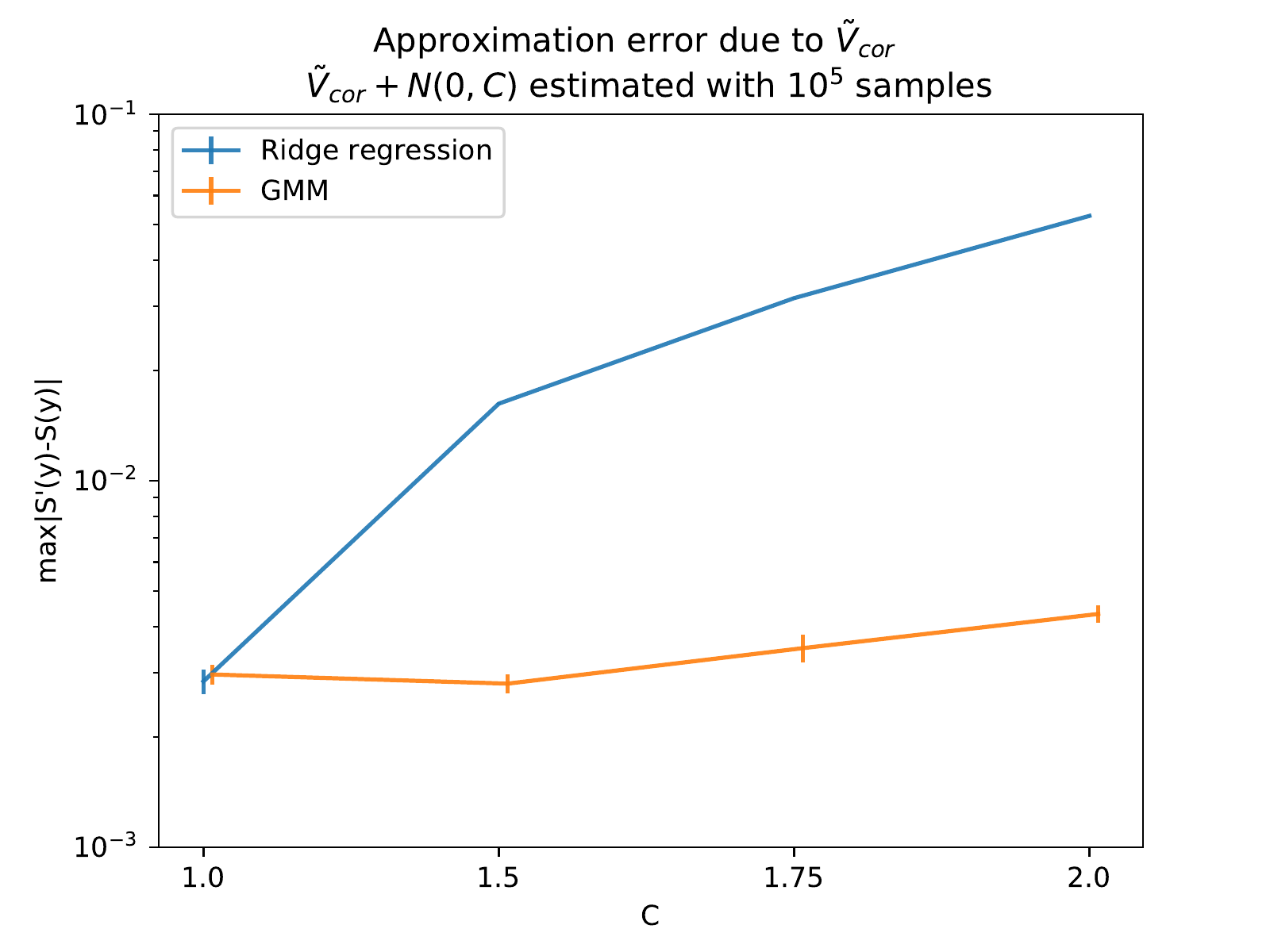}}
\caption{
Approximation error due to $\tilde V_{cor}$ with error bars showing the standard error of the mean 
calculated from 20 runs. With the ridge regression solution proposed by \citet{Seita_2017} the error 
increases quickly when $C>1.$ Using the GMM approximation we can achieve significantly smaller 
error with $C=2.$
}
\label{fig:x_cor_approx_errors}
\end{center}
\vskip -0.2in
\end{figure}

Figure \ref{fig:approximate_pdfs_figure} shows the two approximations with increasing $C.$ When the negative values 
in the ridge regression solution are projected to zeroes, the variance of $V_{cor}$ increases and the resulting approximate $\tilde V_{log}$ 
has variance much larger than the actual $\pi^2/3$ it should have. This also shows in the resulting approximation. 
Figure \ref{fig:approximate_cdfs_figure} shows the empirical cdf for both approximations and for the true logistic distribution, and 
the absolute distance between the approximations $S'$ and the true logistic cdf $S$.

\begin{figure*}[tb]
\subfigure[Ridge regression results]
{\label{fig:seita_approximate_corr_pdfs}
 \includegraphics[width=0.55\textwidth,trim=0mm 0mm 0mm 0mm,clip]{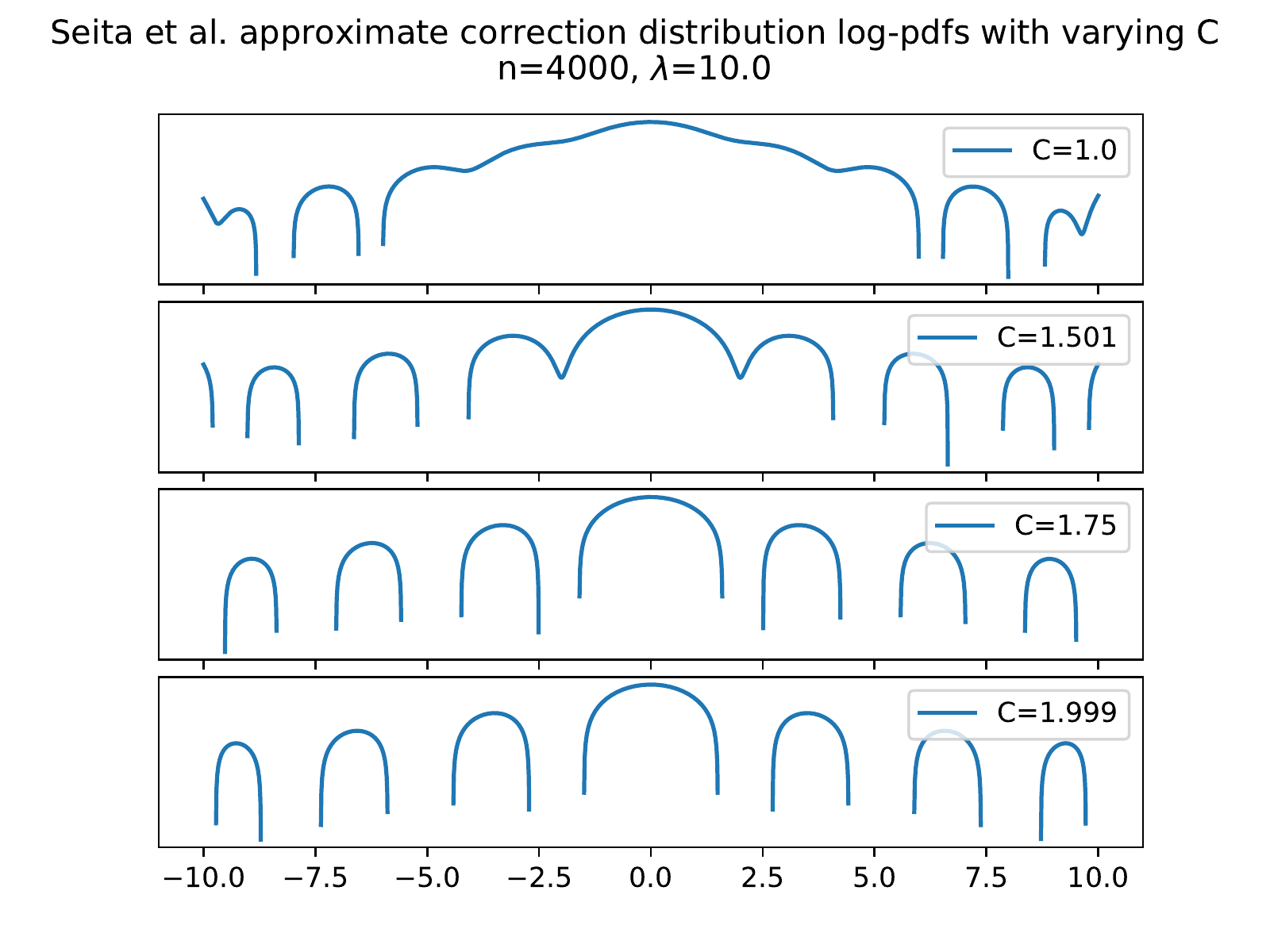}
 }
 \subfigure[GMM results]
{\label{fig:GMM_approximate_corr_pdfs}
 \includegraphics[width=0.55\textwidth,trim=0mm 0mm 0mm 0mm,clip]{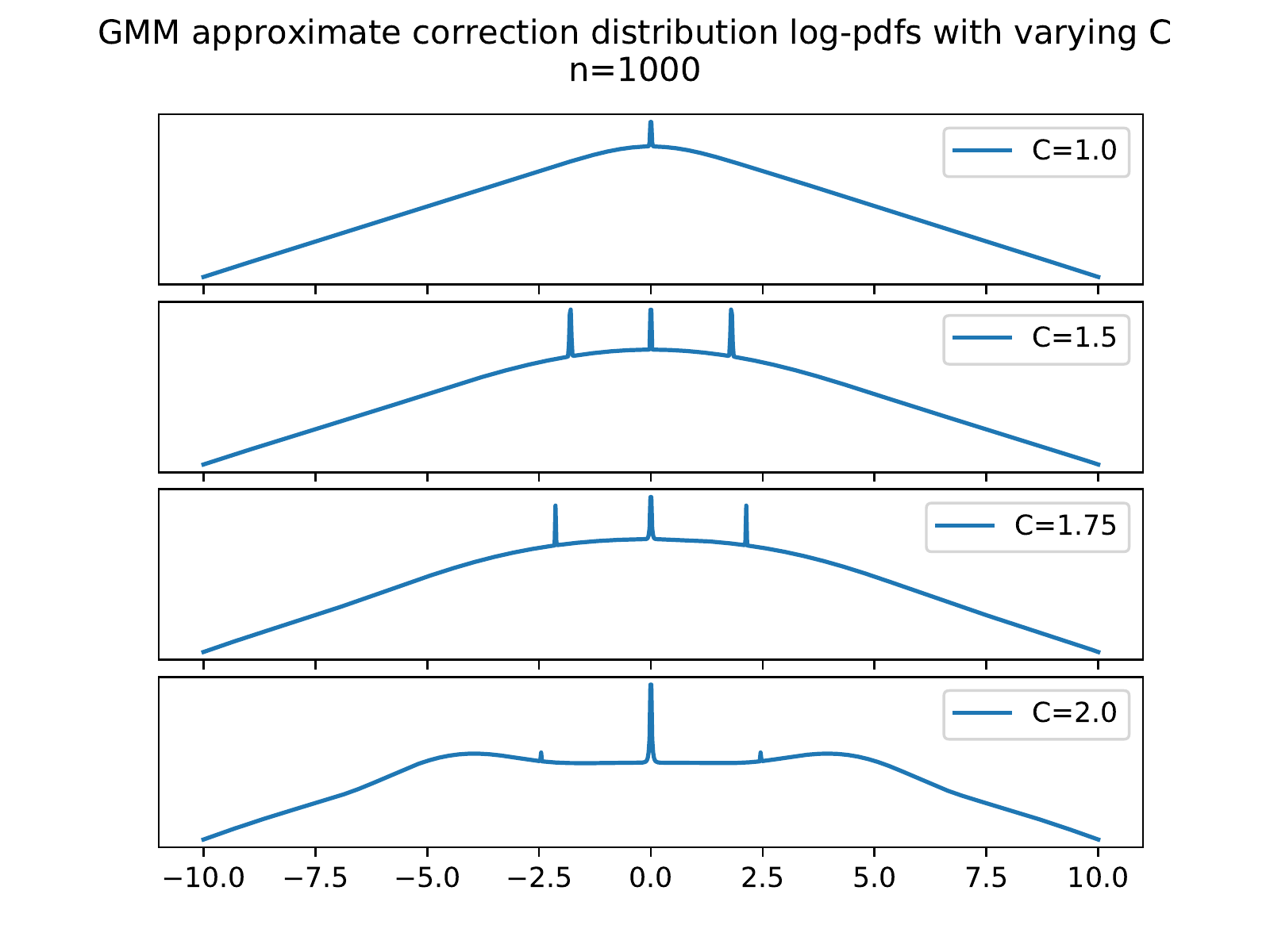} 
 }
  \caption{
  Approximate correction distribution log-densities with varying $C$ values. 
  Figure \ref{fig:seita_approximate_corr_pdfs} shows the results for the ridge regression solution 
  used by Seita et al.: as $C$ increases, the amount of negative values that are projected to zeroes, 
  which show as gaps in the log-pdf, increases markedly.
  Figure \ref{fig:GMM_approximate_corr_pdfs} shows corresponding results for our GMM solution: 
  the approximation is always a valid pdf over $\mathbb R$. 
  }
  \label{fig:approximate_pdfs_figure}
\end{figure*}

\begin{figure*}[tb]
\subfigure[Approximation ecdf and true logistic cdf]
{\label{fig:approximate_cdfs_figure1}
 \includegraphics[width=0.55\textwidth,trim=0mm 0mm 0mm 0mm,clip]{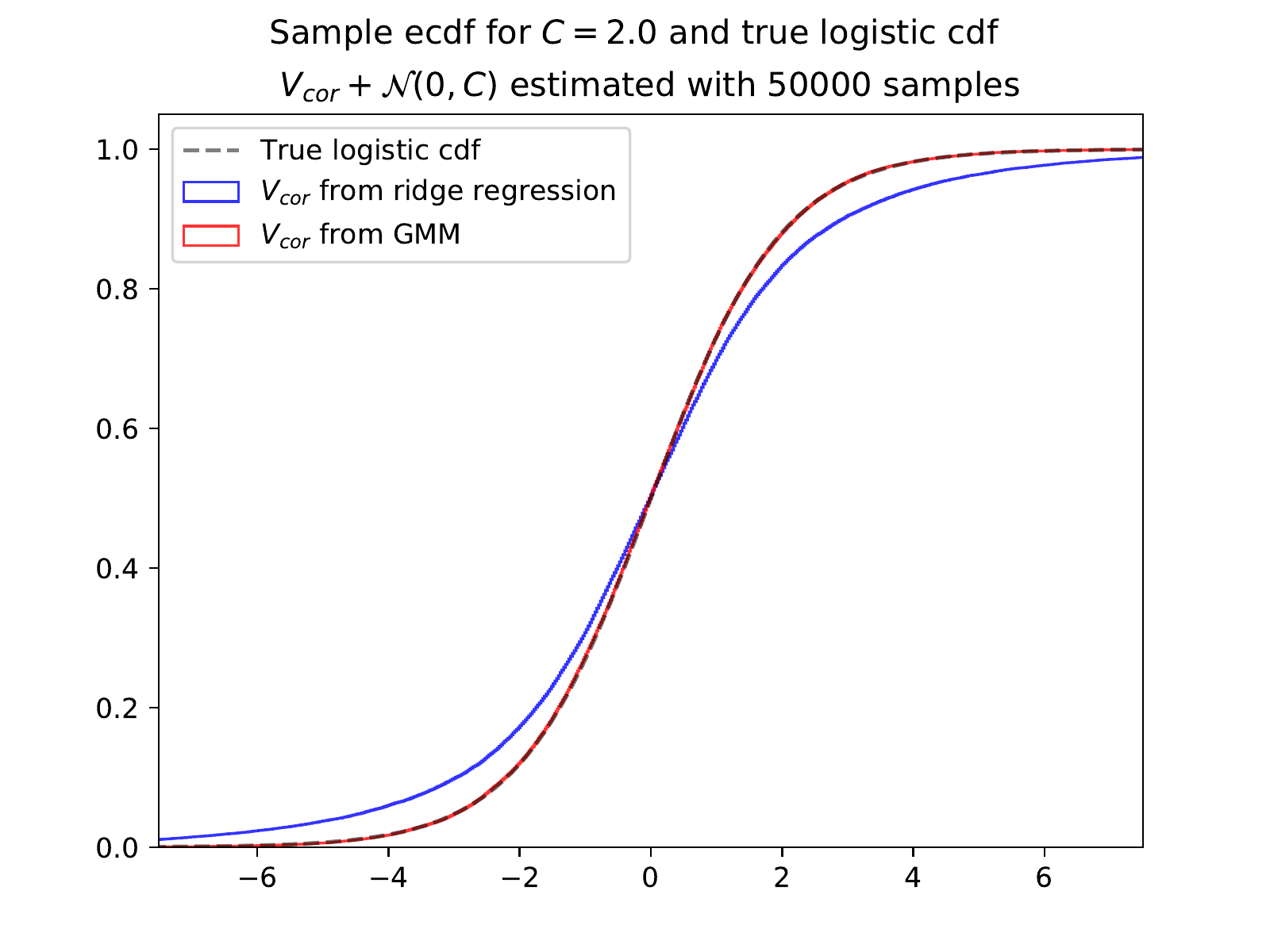}
 }
 \subfigure[Absolute differences from true logistic cdf]
{\label{fig:approximate_cdfs_figure2}
 \includegraphics[width=0.55\textwidth,trim=0mm 0mm 0mm 0mm,clip]{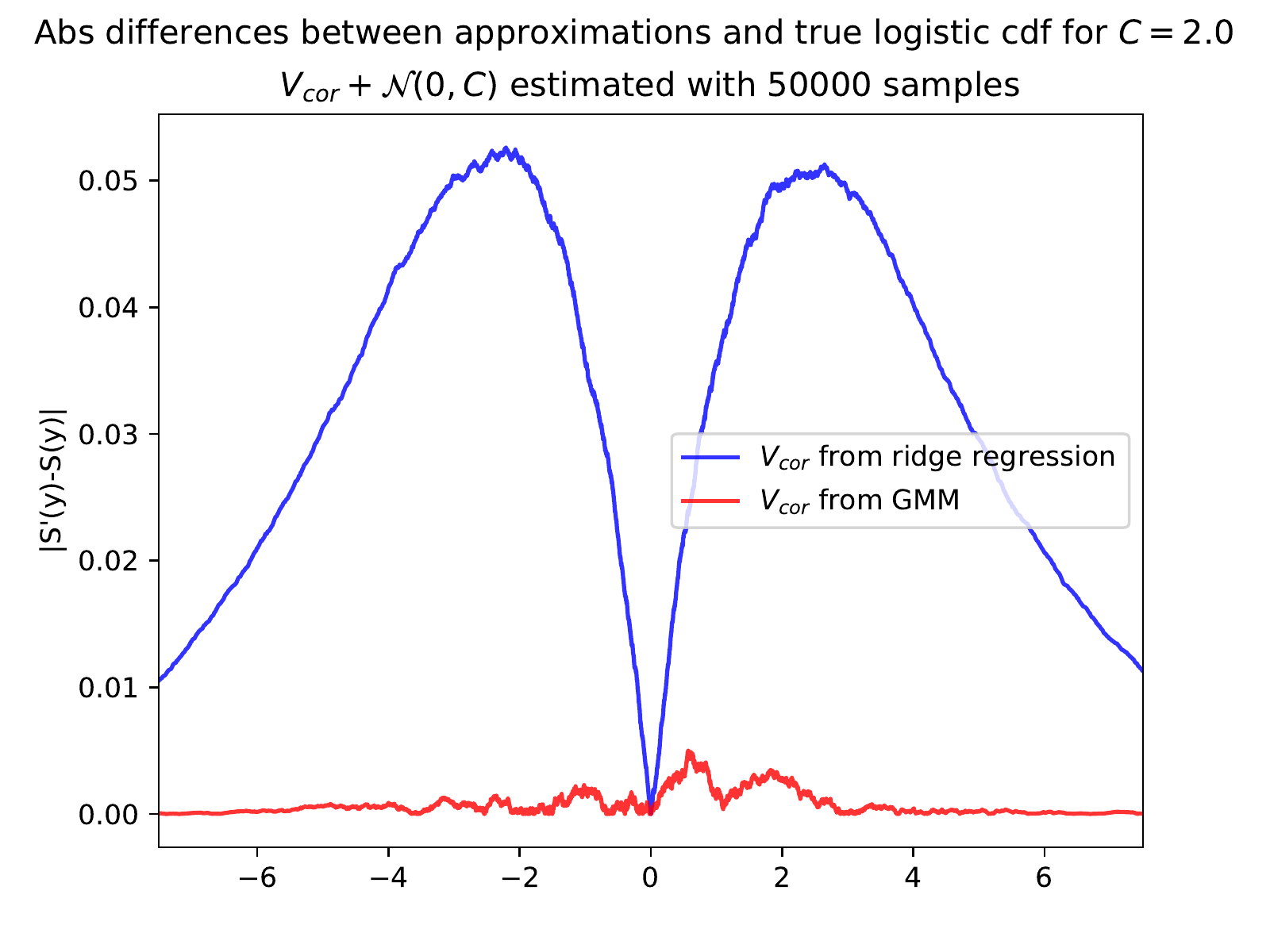} 
 }
  \caption{
  Figure \ref{fig:approximate_cdfs_figure1} shows the empirical cdf for the approximate logistic distributions calculated using  the ridge regression 
  solution of Seita et al. and our GMM together with true logistic cdf. The variance of $V_{cor}$ using ridge regression is 
  too high and the resulting $V_{cor}+\mathcal N(0,C)$ is clearly off. The ecdf for GMM is 
  almost indistinguishable from the true cdf. 
  Figure \ref{fig:approximate_cdfs_figure2} shows the absolute distances between the approximation ecdf and 
  the true logistic cdf.
  }
  \label{fig:approximate_cdfs_figure}
\end{figure*}

To calculate the ridge regression solution for $[-10,10]$, we use the original code of \cite{Seita_2017} with parameter values 
$n=4000, \lambda=10.0$ used in the original paper. The problems with larger $C$ values persisted 
with other parameter settings we tested. Note that the discretisation granularity parameter $n$ used in the 
two methods are not directly comparable.

To fit the GMMs with $K$ components, we take the interval $[-10,10]$ with $n=1000$ points for calculating the loss function, 
and run 20000 optimisation iterations with PyTorch \citep{paszke_2017}. We use Adam optimiser \citep{Kingma_2014} with learning rate $\eta=0.01$ 
and otherwise default settings. The approximation is forced to be symmetric about zero by adding mirrored components: 
for the $k$th component we add a copy but set the mean as $-\mu_k,$ and set the weights as $\pi_k/2$ for both, i.e., use the mean of 
the original and the mirrored component. We use $K=50$ in the test, which gives 100 components with mirroring.

\end{document}